\documentclass{article}
\usepackage[square,numbers]{natbib}
\usepackage{authblk}
\usepackage{charter}
\usepackage{fullpage}
\usepackage[utf8]{inputenc}
\usepackage[T1]{fontenc}
\usepackage{hyperref}
\usepackage{url}
\usepackage{booktabs}
\usepackage{amsfonts}
\usepackage{nicefrac}
\usepackage[protrusion=true,expansion=false]{microtype}
\usepackage{xcolor}
\usepackage{amsmath,amssymb,amsthm,amsfonts}
\usepackage{latexsym,graphicx,float,mathtools}
\usepackage[export]{adjustbox}
\usepackage[noabbrev]{cleveref}
\usepackage{caption}
\usepackage{subcaption}
\usepackage[inline]{enumitem}
\usepackage{multirow}
\usepackage{multicol}
\usepackage{algorithm}
\usepackage{algorithmic}
\usepackage{xspace}
\usepackage{bm}
\usepackage{wrapfig}
\usepackage{tabularx}
\usepackage{array}
\usepackage{colortbl}
\usepackage{arydshln}

\theoremstyle{plain}
\newtheorem{theorem}{Theorem}[section]

\newtheorem{lemma}[theorem]{Lemma}
\newtheorem{corollary}[theorem]{Corollary}
\theoremstyle{definition}

\newtheorem{assumption}[theorem]{Assumption}
\theoremstyle{remark}

\newcolumntype{C}{>{\centering\arraybackslash}X}

\newcommand{\ie}{\emph{i.e.}\xspace}

\newcommand{\sys}{\textsc{TAH-Quant}\xspace}
\renewenvironment{abstract}{%
  \begin{center}{\normalfont\Large\bfseries\abstractname}\end{center}%
  \normalsize
}{}

\title{\sys: Effective Activation Quantization in Pipeline Parallelism over Slow Network}
\author{
    Guangxin He$^{*\dag}$, Yuan Cao$^{*\ddag}$, Yutong He$^{\ddag}$, Tianyi Bai$^{\dag}$\\
    \vspace{-0.5em}
    Kai Chen$^{\dag}$, Kun Yuan$^{\ddag}$, Binhang Yuan$^\dag$\\
    \dag \ HKUST \ \ \ \ \ddag \ Peking University 
}
\date{}

\begin{document}
\maketitle

\vspace{-0.75em}
\vspace{-2ex}
\begin{abstract}
\noindent Decentralized training of large language models offers the opportunity to pool computational resources across geographically distributed participants, but is often bottlenecked by network communication, particularly under pipeline parallel settings. While pipeline parallelism partitions model layers across devices to handle large-scale models, it necessitates frequent communication of intermediate activations, creating challenges when network bandwidth is limited.
To address these issues, we propose \sys (\underline{\textbf{T}}ile-wise \underline{\textbf{A}}daptive \underline{\textbf{H}}adamard \underline{\textbf{Quant}}ization), a novel activation quantization framework for pipeline parallelism. \sys integrates fine-grained tile-wise quantization, entropy-guided tile-wise adaptive bit allocation for optimal bit usage, and a Hadamard-based transformation with pivot swapping to effectively suppress outliers. Compared with token-level allocation, the tile-wise allocator assigns precision at the granularity of small channel windows within each token, reducing quantization error under the same bit budget. We prove that pipeline parallel training equipped with \sys maintains a convergence rate of $\mathcal{O}(1/\sqrt{T})$, matching that of vanilla stochastic gradient descent. Extensive experiments demonstrate that \sys achieves an aggressive activation quantization ratio of 3--4 bits, providing up to $4.3\times$ throughput speedup over uncompressed FP32 and up to $1.33\times$ wall-clock speedup over AQ-SGD, while preserving training convergence, avoiding AQ-SGD's activation-cache overhead, and generalizing well across various training scenarios.
\end{abstract}

\section{Introduction}

Decentralized or open collaborative training of large language models (LLMs) has recently gained significant attention as it enables the pooling of computational resources across multiple geo-distributed participants, thus facilitating the training of models that exceed the capacity of any single resource contributor~\cite{ryabinin2020towards,yuan2022decentralized,gandhi2024improving}. However, a major barrier to these approaches is network communication: unlike specialized clusters equipped with high-speed interconnects, decentralized settings typically rely on slower networks, severely constraining training efficiency~\cite{wang2022fine,wang2023cocktailsgd}. On the other hand, scaling LLM training in the state-of-the-art scale necessitates distributed parallel training --- particularly pipeline parallelism~\cite{huang2019gpipe,narayanan2019pipedream,narayanan2021memory}, which partitions model layers across multiple stages to support training large-scale models. Yet, pipeline parallelism inherently requires frequent transmission of activations and the corresponding gradients between adjacent pipeline stages. In this paper, we explore \textit{how to effectively compress the communication volume to accommodate pipeline parallelism over a slow network.}

Enabling efficient activation compression for pipeline parallelism over slow network connections has significant implications for democratizing large-scale LLM training~\cite{yuan2022decentralized,wang2022fine}. Currently, the capability to train state-of-the-art models remains concentrated among institutions equipped with specialized high-performance computing resources. Effectively addressing network communication bottlenecks would substantially reduce barriers to participation, allowing a broader array of contributors, including universities, startups, and individuals, to collaboratively train or fine-tune LLMs~\cite{douillard2025streaming}.

On the other hand, a significant obstacle arises from the fact that naive activation compression (e.g., quantization~\cite{han2016deep,hubara2017quantized}) methods can negatively affect training convergence. Unlike gradient compression in data parallelism --- where quantization errors typically behave as unbiased noise compatible with optimization procedures, compressing intermediate activations directly influences the neural network's forward computation, consequently introducing bias into gradient estimates in the later backward propagation~\cite{evans2021ac,chakrabarti2019backprop}. Specifically, in pipeline parallel training, compression errors incurred during activation transmission propagate through nonlinear transformations and can distort gradient calculations during the backward pass~\cite{wang2022fine}. Thus, aggressively reducing activation precision without careful management will result in performance degradation or even training divergence.


To restrict the error propagation introduced by activation quantization, prior efforts, such as AQ-SGD~\cite{wang2022fine}, have attempted to address this issue by compressing the changes in activations between training epochs rather than the activations themselves, thereby providing theoretical convergence guarantees leveraging the help from the error compensation. Although effective in preserving model accuracy, AQ-SGD requires storing previous activations for the whole dataset to compute these changes, resulting in substantial memory overhead. Such an approach poses practical limitations, especially in resource-constrained environments for the large-volumes of training data where storage capacity and system complexity are critical considerations.

In this paper, we solve this problem with a new approach for effective activation quantization in pipeline parallelism. In particular, we make the following key contributions:

\textbf{\underline{Contribution 1.}}  We propose \sys, an activation quantization approach to alleviate communication bottlenecks in pipeline-parallel training of LLMs. Specifically, our method includes: (\underline{\textbf{i}}) a fine-grained, tile-wise group quantization technique for localized precision control, effectively limiting quantization error; (\underline{\textbf{ii}}) an entropy-guided, tile-wise adaptive bit allocation method that dynamically assigns precision based on activation distribution characteristics at a granularity that strictly refines token-level allocation, further optimizing the compression efficiency; and (\underline{\textbf{iii}}) a Hadamard-based outlier suppression transform enhanced by a pivot element swap, which effectively mitigates quantization errors arising from extreme activation values. Collectively, these carefully designed techniques enable efficient, accurate low-bit quantization of activations, substantially improving the practicality of decentralized and collaborative LLM training.

\textbf{\underline{Contribution 2.}} We analyze the effect of \sys's activation compression on pipeline-parallel optimization under standard stochastic assumptions, together with an empirically validated condition on the resulting compression error. Concretely, we prove that pipeline parallel training equipped with \sys converges at a rate of $\mathcal{O}(1/\sqrt{T})$, matching that of vanilla SGD.

\textbf{\underline{Contribution 3.}} We conduct extensive experiments on various LLM training tasks, including fine-tuning (\texttt{GPT2-XL}), instruction tuning (\texttt{Qwen2.5-3B}), and from-scratch pretraining (\texttt{LLaMA-3.2-1B}). We show that \sys can aggressively quantize activations to 3--4 bits without sacrificing convergence relative to the state-of-the-art, \ie, \texttt{AQ-SGD}, while introducing no additional activation-cache storage overhead. These results demonstrate that \sys is broadly applicable to different training tasks.

\section{Preliminary and Related Work}

\textbf{Decentralized training of LLM.} Decentralized training of LLMs has garnered significant attention as an interesting attempt to democratize access to large-scale LLM training development~\cite{ryabinin2020towards,borzunov2022training,borzunov2023distributed,gandhi2024improving,blagoev2025skippipe}. Early efforts demonstrated the feasibility of collaborative training across geographically distributed participants with constrained resources under the scope of data parallelism~\cite{diskin2021distributed,borzunov2022training}, where various effective gradient compression methods have been explored~\cite{wang2023cocktailsgd}. To further scale out the training computation, more advanced modes of parallel strategies have been integrated~\cite{yuan2022decentralized,ryabinin2023swarm,lu2024position,strati2024ml}, for example, Yuan et al.~\cite{yuan2022decentralized} addressed the challenges of training foundation models in heterogeneous environments by introducing a scheduling algorithm that optimally allocates computational tasks across decentralized GPUs;
Ryabinin et al.~\cite{ryabinin2023swarm} proposed SWARM Parallelism, where temporary randomized pipelines between nodes are adaptively rebalanced to handle dynamic efficient training of large Transformer models using preemptive instances with limited network bandwidth. Douillard et al.~\cite{douillard2025streaming} improve DiLoCo with sequential synchronization, comm--compute overlap, and quantized exchange. Most recently, Ramasinghe et al.~\cite{ramasinghe2025subspace} propose \emph{Subspace Networks}, which reduce model-parallel activation traffic by constraining projection weights to a shared low-rank subspace. This direction is complementary to \sys: Subspace Networks reduce communication through weight-parameterization constraints, whereas \sys is a drop-in activation-compression module that leaves the model parameterization unchanged.

\textbf{Activation compression in training.} Activation compression techniques have been studied to reduce memory and computational overhead in model training~\cite{liu2021exact,bersatti2020neural,georgiadis2019accelerating,fu2020don,liu2022gact,chendropit,bian2024does}. Concretely, inherent sparsity in activation has been studied to minimize storage and computation in neural networks. For example, Zhang et al.~\cite{zhang2024exploring} investigate the natural occurrence of sparse activations in pretrained Transformers and dynamically alternate between sparse and dense training phases to enhance pretraining efficiency~\cite{rhu2018compressing,jiang2022back,zhang2024exploring,lilazy}.  On the other hand, quantization-based methods~\cite{evans2020jpeg,liu2021exact,wang2023division} reduce the precision of activations to lower bit-widths, thereby decreasing memory usage. For example, Han et al.~\cite{han2016deep} presented Deep Compression, combining pruning, trained quantization, and Huffman coding. Hubara et al.~\cite{hubara2017quantized} explore training neural networks with low-precision weights and activations. Chakrabarti et al.~\cite{chakrabarti2019backprop} propose backpropagation with approximate activations for memory-efficient training. Chen et al.~\cite{chen2021actnn} introduced ActNN, employing 2-bit activation compressed training.

\textbf{Quantization for LLM.} Quantization has emerged as a key technique for serving LLMs efficiently by reducing the precision of weights~\cite{lin2024awq,frantar2022gptq}, activations~\cite{xiao2023smoothquant}, and KV-cache~\cite{liu2024kivi} for the process of generative inference, parameter-efficient fine-tuning~\cite{dettmers2023qlora}, and large-scale pretraining~\cite{you2024parameterize,liu2024llm}. For example, AWQ~\cite{lin2024awq} quantizes the LLM weights by identifying a small subset of ``salient'' weight channels and scales them up before quantization, thereby preserving accuracy even at 4-bit weight precision. KIVI~\cite{liu2024kivi} proposes a tuning-free 2-bit quantization of the KV cache (with per-channel asymmetric scaling), dramatically reducing memory and enabling longer context lengths with negligible impact on generation quality. QLoRA~\cite{dettmers2023qlora} demonstrated that a 4-bit quantized base model can be \emph{fine-tuned} via low-rank adapters to reach the same performance as full \texttt{FP16} fine-tuning. LLM-QAT~\cite{liu2024llm} introduces a data-free QAT scheme that allows 4-bit quantization of weights, activations, and even the KV cache while preserving performance for training. One essential problem in such quantization methods is how to effectively resolve the issues of outliers in the quantization group~\cite{lin2024duquant,hu2025ostquant,you2024parameterize}. An especially simple yet effective transformation for outlier suppression is the \textit{Hadamard transform}~\cite{theodoridis2009feature}. Formally, the $N \times N$ Hadamard matrix $\mathbf{H}_N \in \{\pm 1\}^{N\times N}$ is defined such that $\mathbf{H}_N \mathbf{H}_N^T = N\mathbf{I}_N$ (so $\frac{1}{\sqrt{n}}\mathbf{H}_N$ is an orthonormal matrix). Multiplying a vector with dimension $N$ by $\mathbf{H}_N$ will evenly redistribute the vector's components across $N$ dimensions.
Recent quantization studies leverage the \textit{Hadamard transform} to suppress outliers and reduce quantization error, mainly for tensor-wise weight quantization in generative inference scenarios, such as QuaRot~\cite{ashkboos2024quarot} and SpinQuant~\cite{liu2024spinquant}. In contrast, we apply lightweight Hadamard transforms for \emph{activation communication} in pipeline-parallel training at a finer tile-wise granularity.

\section{Activation Quantization}

Pipeline-parallel training~\cite{huang2019gpipe,narayanan2019pipedream, narayanan2021memory} requires communication of intermediate activations between devices. This communication can become a key bottleneck on slow interconnects.
To alleviate this by quantization-based compression, we leverage three carefully-designed mechanisms \textit{in order to reduce the quantization error}, including: (\underline{\textbf{i}}) fine-grained tile-wise group quantization for localized precision control (Section~\ref{sec:tile}); (\underline{\textbf{ii}}) an entropy-guided tile-wise adaptive bit-width allocation (Section~\ref{sec:adap}); and (\underline{\textbf{iii}}) a Hadamard-transform-based outlier suppression with a pivot element swap (Section~\ref{sec:tran}). We also discuss how we integrate the proposed \sys quantization method in pipeline parallel training in Section~\ref{sec:whole}. We detail them below.

\subsection{Fine-Grained Tile-Wise Group Quantization}
\label{sec:tile}

First, we introduce a fine-grained, tile-wise group quantization scheme for localized precision control. Specifically, instead of quantizing the entire activation tensor with a single set of parameters, we partition it into small tiles and quantize each tile independently. For example, consider an activation tensor $\mathbf{a}$ of shape $B \times S \times C$, i.e., $\mathbf{a} \in \mathbb{R}^{B \times S \times C}$, where $B$, $S$, and $C$ denote the batch size, sequence length, and number of channels (i.e., model dimension), respectively. We partition this tensor along the channel dimension into multiple tiles by grouping contiguous channels within each token. Each such tile (i.e., quantization group) can be noted as $\mathbf{a}_{i,j,t}\in  \mathbb{R}^{G}$, where $G$ is the quantization group size determined by $G=\frac{C}{N_t}$, $N_t$ is the number of partitions of all the channels, and $i=1,\ldots, B$, $j=1,\ldots, S$, $t=1,\ldots, N_t$ are the indices for each tile-wise quantization group. Note that each tile will form a separate quantization group with its own scale and zero-point. This approach ensures that each group is quantized using an optimal dynamic range, improving low-bit accuracy. By confining quantization error to small groups, we avoid the coarse tensor-wise scale being dominated by a few extreme values.

\subsection{Tile-Wise Adaptive Bit Allocation}
\label{sec:adap}

The fine-grained grouping addresses local range variation; however, even within a single token the activation values may have non-uniform importance: some contiguous channel windows are smooth while others contain sharp peaks. We therefore introduce an entropy-based, tile-wise adaptive precision allocation strategy that dynamically adjusts the quantization bit width at the granularity of small channel-window allocation tiles within each token rather than entire tokens.

\textbf{Tile decomposition.}
Let $A\in\mathbb{Z}_{>0}$ denote the allocation-tile size, with $A\le C$ and $A$ chosen to align with the quantization tile boundaries of Section~\ref{sec:tile}. Each token is split into $M=C/A$ non-overlapping channel-window allocation tiles. We reshape the activation tensor
\[
\mathbf{a} \in \mathbb{R}^{B\times S\times C}
\;\longmapsto\;
\tilde{\mathbf{a}} \in \mathbb{R}^{B\times (SM)\times A},
\]
so that each row $\tilde{\mathbf{a}}_{i,j'} \in \mathbb{R}^{A}$ represents one allocation tile, indexed by $j' = 1,\dots, SM$, that combines a token index and a channel-window index. Setting $A=C$ (so $M=1$) recovers the previously studied token-level allocation as a strict special case; choosing $A<C$ yields finer precision control. In our default configuration, $A=G$, so each allocation tile is also a quantization group.

\textbf{Per-tile entropy.}
For every allocation tile $\tilde{\mathbf{a}}_{i,j'}\in\mathbb{R}^A$ we form its normalized magnitude distribution
\begin{equation}
    p_k \;=\; \frac{|\tilde{a}_{i,j',k}|}{\|\tilde{\mathbf{a}}_{i,j'}\|_1+\epsilon}, \quad k=1,\ldots, A,
\end{equation}
where $\epsilon$ is a small positive constant for numerical stability, and define
\begin{equation}
\mathcal{H}\big(\tilde{\mathbf{a}}_{i,j'}\big) \;=\; \sum_{k=1}^A p_k \log\big(p_k+\varsigma\big),
\end{equation}
with $\varsigma$ a small positive constant avoiding zero arguments to the logarithm. As before, $\mathcal{H}$ is high when energy is spread evenly across the $A$ channels and low when the tile contains a sharp outlier.

\textbf{Top-$p\%$ tile selection.}
We rank the $SM$ allocation tiles of every sample by their entropy and assign \texttt{INT4} (high precision) to the top-$p\%$ \emph{highest-entropy} tiles and \texttt{INT3} (low precision) to the rest. The intuition is unchanged from the token-level argument: high-entropy tiles lack a single channel that the Hadamard transform of Section~\ref{sec:tran} can isolate, so they need extra precision; low-entropy tiles concentrate energy on a few channels and thus tolerate aggressive quantization \emph{after} the outlier-suppression transform. The resulting 1-bit per-tile bit map is packed with the quantized payload as metadata.

\textbf{Why tiles rather than whole tokens.}
Empirically we observe that, within a single token, the entropy of different channel-window tiles can differ by more than an order of magnitude --- some tiles contain a strong outlier while others are perfectly smooth. Whole-token allocation forces both kinds of tiles to share the same bit budget, wasting bits on smooth regions and starving spiky regions. Tile-wise allocation avoids this mismatch: a token may now be encoded as a heterogeneous mixture of \texttt{INT4} and \texttt{INT3} tiles. Since tile-wise allocation reduces to token-level allocation when $A=C$ (every allocation tile is a whole token), the assumption used by our convergence proof (Assumption~\ref{asp:contractive}) carries over verbatim --- in fact the per-bit quantization error can only decrease.

\subsection{Hadamard-Based Outlier Suppression Transform}
\label{sec:tran}

Outliers in the activation values can severely degrade the accuracy of low-bit quantization even within a small quantization group. To mitigate quantization error caused by extreme outliers in activation groups, we propose an adaptive Hadamard transform strategy, which consists of three steps: (\underline{\textbf{i}}) a \textit{heuristic-based outlier detection} to decide if transform is needed, (\underline{\textbf{ii}}) a \textit{Hadamard transform with pivot element swap} to redistribute the outlier values in the quantization group, and (\underline{\textbf{iii}}) an \textit{asymmetric uniform quantization} of the values in the quantization group.

\textbf{Outlier detection heuristic}: Given any quantization group $\mathbf{a}_{i,j,t} = [a^{i,j,t}_1, a^{i,j,t}_2, \dots, a^{i,j,t}_G] \in \mathbb{R}^G$, where $G=\frac{C}{N_t}$ is the quantization group size. For the rest parts in Section~\ref{sec:tran}, we simplify the notation as $\mathbf{a}_{i,j,t} = \bm{\alpha} = [\alpha_1, \alpha_2, \dots, \alpha_G] \in \mathbb{R}^G$ to introduce the quantization method within each tile. In order to detect whether an outlier is present, we define the following heuristic:
\begin{equation}\label{eq:tau}
    r = \frac{|\alpha^{(1)}|}{|\alpha^{(2)}|+\varrho}
\end{equation}
Where $\alpha^{(1)}$ and $\alpha^{(2)}$ represent the elements in $\bm{\alpha}$ with the largest and the second largest absolute values, $\varrho$ is a small positive constant. If $r$ exceeds a threshold $\tau$ (empirically, we set $\tau = 2.0$), we will deem $\mathbf{a}_{i,j,t}$ to contain an outlier and apply the Hadamard-based transform as we will introduce below; otherwise, we skip this transform for that tile. We choose $\tau$ via a sensitivity study: $\tau=2.0$ consistently improves convergence compared to always applying ($\tau=0$) or never applying ($\tau=\infty$) the transform; detailed results are provided in Appendix~\ref{app:tau}.

\textbf{Hadamard transform with pivot element swap}: For a group identified to have an outlier, we perform a \textit{pivot element swap} to align the pivot (the element with the largest absolute value) with the Hadamard matrix structure. Let $d = \arg\max_k |\alpha_k|$ denote the pivot element index (i.e., $\alpha_d=\alpha^{(1)}$). We define a permutation matrix $\mathbf{P}_d \in \mathbb{R}^{G \times G}$ that swaps the first and $d$-th coordinates, which yields a permuted vector by multiplying this permutation matrix:
$$[\alpha_d, \alpha_2, \dots, \alpha_1, \ldots, \alpha_G] = [\alpha_1, \alpha_2, \dots, \alpha_G] \mathbf{P}_d =  \bm{\alpha}\mathbf{P}_d $$
Next, we multiply this transformed vector by a Hadamard matrix $\mathbf{H}_G \in \{\pm 1\}^{G \times G}$ to redistribute the values and resolve the issue of outliers:
\begin{equation}
    \dot{\bm{\alpha}} = \bm{\alpha}\mathbf{P}_d \frac{1}{\sqrt{G}}\mathbf{H}_G
\end{equation}
After applying this transform, the extreme value in the original $\bm{\alpha}$ will be redistributed across all components in the transformed vector $\dot{\bm{\alpha}}$. This transform greatly reduces the dynamic range of the group: the formerly pivot value is no longer isolated in a single position, yielding a more balanced tile for the activation vector. As a result, the quantization error can be reduced, since a tighter quantization scale can represent the values with higher precision. Notably, because $\mathbf{H}_G$ is orthogonal (i.e., $\mathbf{H}_G \mathbf{H}_G^T = G\mathbf{I}_G$), we can later invert the transform by applying $\mathbf{H}_G^T$ to the de-quantized values when recovery of the original domain is required.

\textbf{Asymmetric uniform quantization}: After the above two steps, the activation values should be uniformly distributed and centered if the outlier issue once existed. Thus, we can apply the a standard asymmetric quantizer (i.e.,~\cite{you2024parameterize}) --- if the computed heuristic $r \le \tau$, we apply this quantizer for the original vector $\bm{\alpha}$; otherwise, we apply this quantizer for the transformed vector $\dot{\bm{\alpha}}$.

\subsection{Computational Overhead}

Following notations in Sections~\ref{sec:tile} and~\ref{sec:adap}, we analyze the computational overhead theoretically and empirically. Pivot swapping selects the top-2 entries per quantization tile and performs a conditional swap, costing $\mathcal{O}(G)$ per tile and $\mathcal{O}(BSN_tG)$ overall (implemented via parallel tensor ops). Entropy computation reduces over $A$ channels per allocation tile, giving $\mathcal{O}(BSC)$ in aggregate, while tile-wise bit allocation ranks $SM=S\cdot C/A$ entropies per sample, costing $\mathcal{O}(B\,SM\log(SM))$, which remains negligible relative to attention and FFN even for small allocation tiles. For outlier tiles, the Hadamard transform applies a $G\times G$ multiplication, costing $\mathcal{O}(G^2)$ per quantization tile and at most $\mathcal{O}(BSN_tG^2)$. Under the default 80\% \texttt{INT4} + 20\% \texttt{INT3} setting with $A=G=64$, \sys transmits about $4.41$ bits per activation element, including roughly $0.61$ bits for scale, zero-point, and bit-map metadata. Heterogeneous precisions and metadata are packed into a contiguous \texttt{uint8} byte stream before communication, so the communication layer observes a fixed-shape buffer rather than variable-length tensors. In contrast, training step time is dominated by attention $\mathcal{O}(BSC(S+C))$ and FFN $\mathcal{O}(BSC^2)$. Empirically, profiling a 4-stage pipeline with micro-batch size 2 shows \sys adds only $\sim 1\%$ overhead with the default $A=G=64$.

\subsection{\sys in Pipeline Parallel Training}
\label{sec:whole}

Given the carefully designed \sys quantization method, it is straightforward to integrate it into standard pipeline-parallel training. We illustrate this process in Algorithm~\ref{alg:pp}. For clarity, we present a two-stage pipeline, which can be easily extended to an arbitrary number of stages. Following \texttt{AQ-SGD}, we use a simple fixed-point (naive) compressor for gradients in the backward pass. This choice is motivated by system efficiency rather than an algorithmic limitation. Backward propagation typically incurs substantially more computation than forward propagation, providing more opportunity for computation--communication overlap, so a higher-bit naive quantizer (e.g., 6--8 bits) is usually sufficient in practice. We emphasize that \sys also applies to backward gradients under ultra-low precision. An empirical comparison is provided in Appendix~\ref{app:backward}.

\begin{algorithm}[H]
   \caption{\sys in a two-stage pipeline parallel training.}
   \label{alg:pp}
\begin{algorithmic}[1]\label{alg_1}
   \STATE {\bfseries Initialize:} sub-network $a(-)$ weights $\mathbf{x}^{(a)}$, sub-network $b(-)$ weights $\mathbf{x}^{(b)}$, optimizer $\rho$.
    \FOR{t = 1, \ldots, T}
        \STATE Randomly sample training batch $\xi_t$.\\
        \textcolor{green}{// Forward propagation:}
        \STATE Machine $a$ sends the quantized output activations $Q_{\textsc{TAH-Quant}}\left(a(\xi_t, \mathbf{x}_{t}^{(a)})\right)$ to Machine $b$.
        \STATE Machine $b$ dequantizes the received activation $Q_{\textsc{TAH-Quant}}\left(a(\xi_t, \mathbf{x}_{t}^{(a)})\right)$.\\
        \textcolor{green}{// Backward propagation:}
        \STATE Machine $b$ sends the quantized gradients w.r.t the activations $Q_{\textsc{Naive}}\left(\nabla_{a} (F\circ b)\vert_{\xi_t}\right)$ back to Machine $a$.
        \STATE Machine $a$ dequantizes the received gradient w.r.t the activations $Q_{\textsc{Naive}}\left(\nabla_{a} (F\circ b)\vert_{\xi_t}\right)$.\\
        \textcolor{green}{// Parameter updates:}
       \STATE  Machine $a$ updates its parameters with gradients $\hat{\mathbf{g}}^t\left(\mathbf{x}^{(a)}\right)$ using optimizer $\rho$.
       \STATE Machine $b$ updates its parameters with gradients $\hat{\mathbf{g}}^t\left(\mathbf{x}^{(b)}\right)$ using optimizer $\rho$.
    \ENDFOR
    \STATE {\bfseries Output:} $\mathbf{x} = (\mathbf{x}_{T}^{(a)}, \mathbf{x}_{T}^{(b)})$
\end{algorithmic}
\end{algorithm}

\begin{figure}[h]
	\centering
        \begin{subfigure}[b]{0.235\textwidth}
		\includegraphics[width=\textwidth]{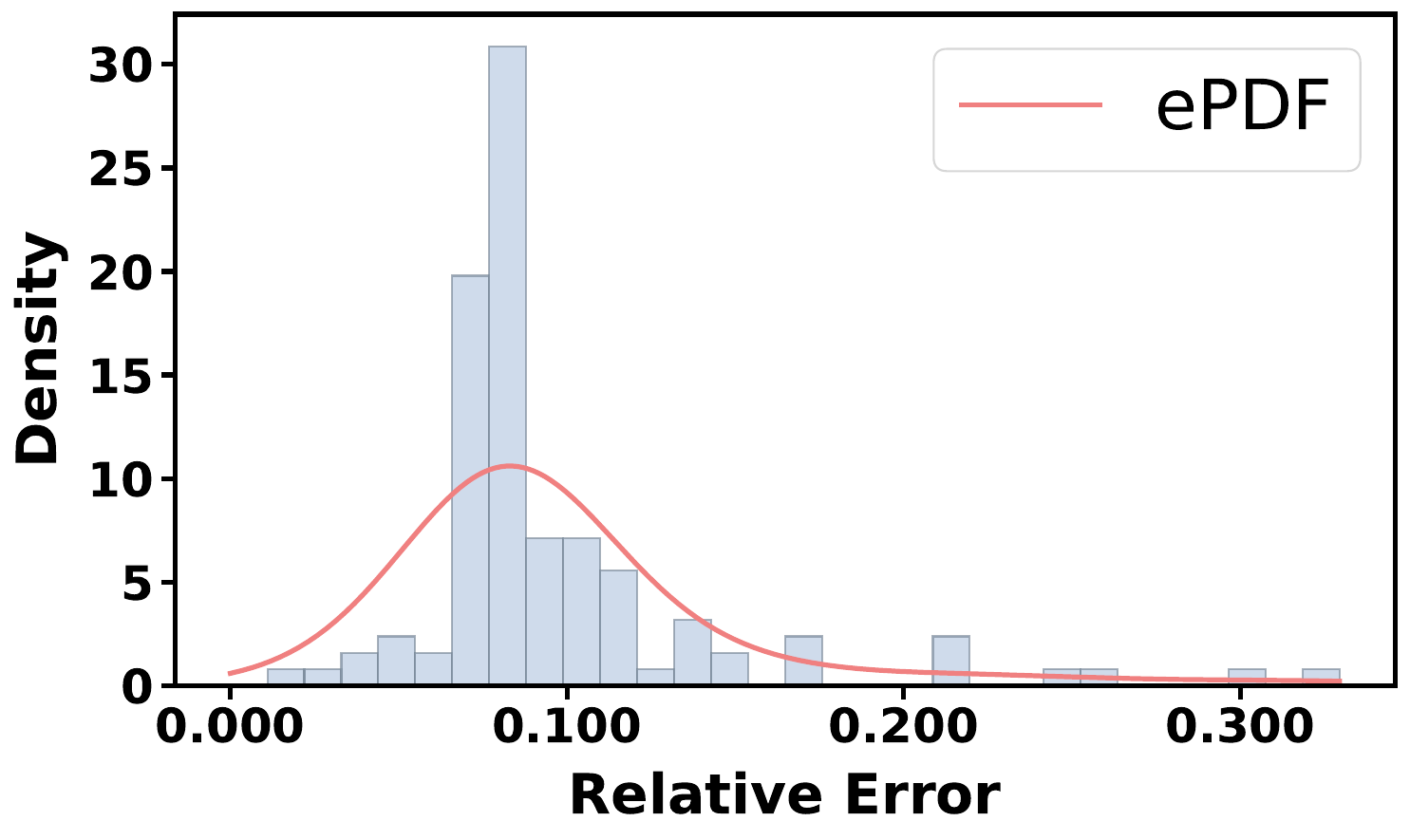}
	\caption{step wise; tile size 64.}
	\label{fig:cgk1}
        \end{subfigure}
	\begin{subfigure}[b]{0.235\textwidth}
        \includegraphics[width=\textwidth]{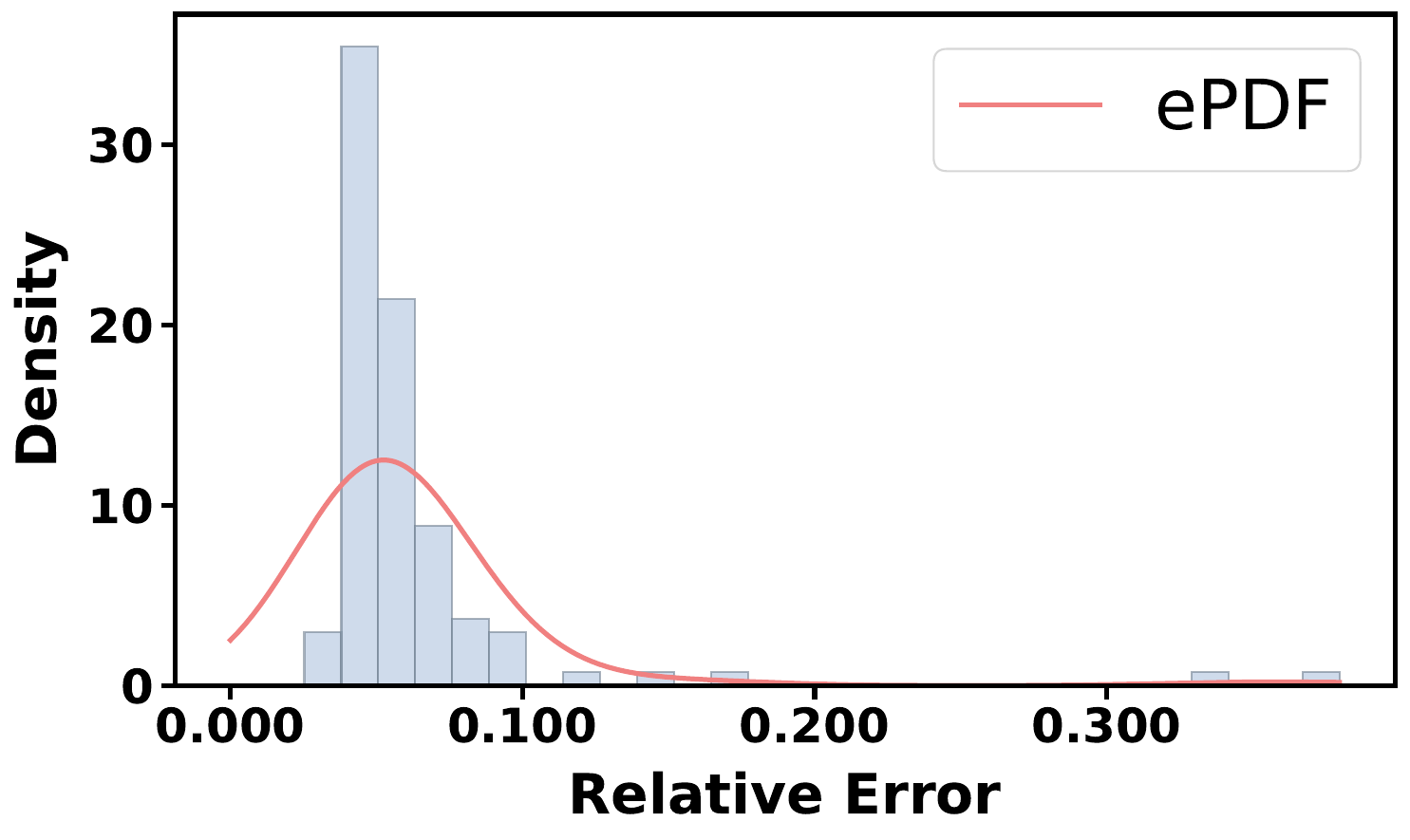}
	\caption{step wise; tile size 32.}
	\label{fig:cgk2}
        \end{subfigure}
         \begin{subfigure}[b]{0.235\textwidth}
		\includegraphics[width=\textwidth]{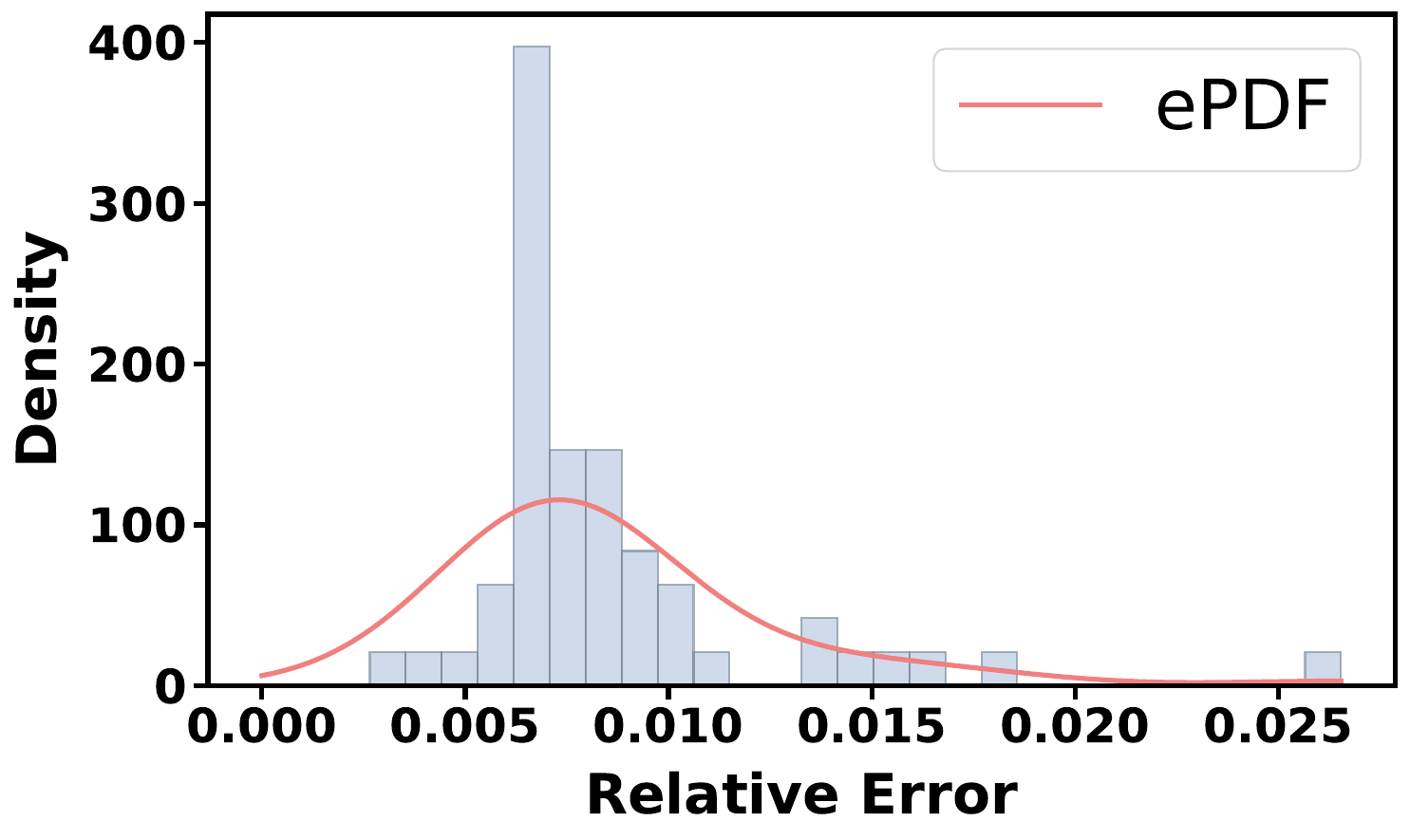}
	\caption{full dataset; tile size 64.}
	\label{fig:ecgk1}
        \end{subfigure}
         \begin{subfigure}[b]{0.235\textwidth}
		\includegraphics[width=\textwidth]{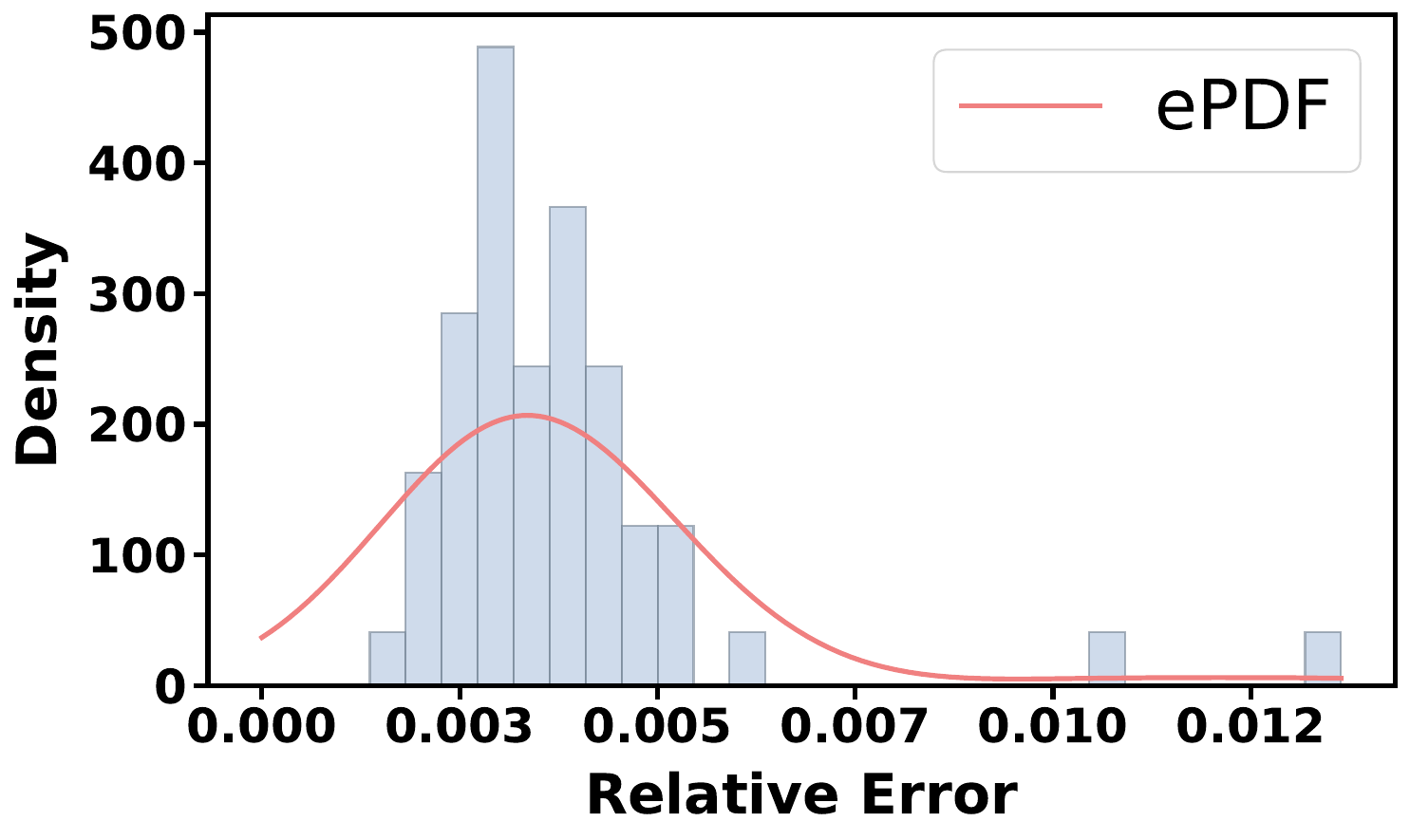}
	\caption{full dataset; tile size 32.}
	\label{fig:ecgk2}
        \end{subfigure}
        \caption{Empirical justification of Assumption~\ref{asp:contractive}.}
\end{figure}

\section{Theoretical Analysis}

In this section, we present convergence guarantees for \sys, which aims to solve the following stochastic optimization problem in a pipeline-parallel fashion:
\begin{align}
\label{prob}
    \min_{\mathbf{x} \in \mathbb{R}^d} \quad \mathbb{E}_{\xi \in \mathcal{D}}[F(\mathbf{x}; \xi)]
\end{align}
where $\mathbf{x}$ denotes the model weights distributed across different pipelines, and $\xi$ represents random data drawn from the distribution $\mathcal{D}$. We denote $\nabla F(\mathbf{x}; \xi)$ as the stochastic gradient and $\nabla f(\mathbf{x})$ as the full-batch gradient. Without loss of generality, we consider using momentum SGD as the optimizer $\rho$ in Algorithm~\ref{alg:pp}:
\begin{align}
    \mathbf{m}^t &= (1-\beta_1)\mathbf{m}^{t-1}+\beta_1\hat{\mathbf{g}}^t, \\
    \mathbf{x}^{t+1} &= \mathbf{x}^t-\eta \mathbf{m}^t,
\label{adaptrule}
\end{align}
where $\beta_1 \in (0,1)$ is the momentum coefficient and $\eta$ is the learning rate. The vector $\hat{\mathbf{g}}^t$ is a quantized estimate of the stochastic gradient $\nabla F(\mathbf{x}^t; \xi^t)$, obtained through Lines 3--7 of Algorithm~\ref{alg_1}. Specifically, it takes the form $\left( \hat{\mathbf{g}}^t(\mathbf{x}^{(a_1)}), \hat{\mathbf{g}}^t(\mathbf{x}^{(a_2)}), \dots, \hat{\mathbf{g}}^t(\mathbf{x}^{(a_N)}) \right)$, where $a_1, a_2, \dots, a_N$ index the machines in the pipeline-parallel system.
Our analysis can be extended to the Adam optimizer with a few more involved derivations.

\subsection{Assumptions}

\begin{assumption}[Lower Boundedness]\label{asp:proper}
	The loss function $f:\mathbb{R}^{d}\rightarrow\mathbb{R}$ satisfies $\inf_{\mathbf{x}\in\mathbb{R}^d}f(\mathbf{x})>-\infty$.
\end{assumption}
\begin{assumption}[$L$-Smoothness]\label{asp:smoothness}
	The loss function $f$ is $L$-smooth, \ie, it holds for any $\mathbf{x},\mathbf{y}\in\mathbb{R}^{d}$ that
	\begin{align*}
		\|\nabla f(\mathbf{x})-\nabla f(\mathbf{y})\|_2\le L\|\mathbf{x}-\mathbf{y}\|_2.
	\end{align*}
\end{assumption}
\begin{assumption}[Stochastic Gradient]\label{asp:stochastic}
We assume that for some $\sigma>0$, the stochastic gradient oracle satisfies
\begin{equation}\label{eq:sg}
\begin{aligned}
\mathbb{E}\big[\nabla F(x^t;\xi^t)\big] &= \nabla f(x^t), \\
\mathbb{E}\big[\|\nabla F(x^t;\xi^t)-\nabla f(x^t)\|^2\big] &\le \sigma^2 .
\end{aligned}
\end{equation}
\end{assumption}

Assumptions~\ref{asp:proper}--\ref{asp:stochastic} are standard assumptions commonly used in stochastic optimization. The following assumption states that gradient quantization through \sys proposed in Algorithm~\ref{alg_1} does not introduce significant distortion to the true stochastic gradient.

\begin{assumption}[Quantization Error]\label{asp:contractive}
	Let $\mathbf{g}^t$ denote the original stochastic gradient $\nabla F(\mathbf{x}^t,\xi^t)$, and $\hat{\mathbf{g}}^t$ denote the quantized stochastic gradient obtained through \sys.
    For some $\delta\in(0,1],$ it holds that
	\begin{align}
		\|\hat{\mathbf{g}}^t-\mathbf{g}^t\|^2&\le(1-\delta)\|\mathbf{g}^t\|^2,\label{eq:asp-cgk} \\
        \|\mathbb{E}_{\xi^t\sim\mathcal{D}}[\hat{\mathbf{g}}^t]-\nabla f(\mathbf{x}^t)\|^2&\le(1-\delta)\|\nabla f(\mathbf{x}^t)\|^2,\label{eq:asp-ecgk}
	\end{align}
\end{assumption}
The above assumption ensures that the quantized gradient $\hat{\mathbf{g}}$ remains close to the true gradient $\mathbf{g}$, with their closeness measured by the quantization coefficient $\delta$. A larger $\delta$ (\ie, $\delta \to 1$) indicates a smaller quantization error. When $\delta = 1$, we have $\hat{\mathbf{g}} = \mathbf{g}$, implying no quantization error.

\textbf{Empirical justification of Assumption~\ref{asp:contractive}.} We now empirically verify that \sys satisfies Assumption~\ref{asp:contractive}. To validate inequality~\eqref{eq:asp-cgk}, we conduct fine-tuning experiments on the Gemma2-2B model using the Math-7K dataset. At each training step, we compute the relative error $\|\hat{\mathbf{g}}^t - \mathbf{g}^t\|^2 / \|\mathbf{g}^t\|^2$, as shown in Figures~\ref{fig:cgk1} and \ref{fig:cgk2}. The results indicate that the relative errors remain below 0.4 across all steps, confirming the validity of \eqref{eq:asp-cgk} with $\delta = 0.6$. To validate inequality~\eqref{eq:asp-ecgk}, we conduct experiments on the same model and dataset. At each step, we compute both the expected compressed gradient $\mathbb{E}_{\xi^t \sim \mathcal{D}}[\hat{\mathbf{g}}^t]$ and the full-batch gradient $\nabla f(\mathbf{x}^t)$, and then evaluate the relative error $\|\mathbb{E}_{\xi^t \sim \mathcal{D}}[\hat{\mathbf{g}}^t] - \nabla f(\mathbf{x}^t)\|^2 / \|\nabla f(\mathbf{x}^t)\|^2$, as shown in Figures~\ref{fig:ecgk1} and \ref{fig:ecgk2}. All relative errors are below 0.1, confirming the validity of \eqref{eq:asp-ecgk} with $\delta = 0.9$. In both experiments, we use tile sizes of 64 and 32, with 80\% \texttt{INT4} and 20\% \texttt{INT3} quantization. These experiments demonstrate the effectiveness of \sys, which quantizes variables to smaller sizes without incurring significant errors.

\subsection{Convergence Guarantees}

Under the above assumptions, we are ready to provide convergence guarantees of our proposed \sys method.

\begin{theorem}\label{thm:quantize}
	Under Assumptions~\ref{asp:proper}--\ref{asp:contractive}, if $\delta\in(0,1)$, $\beta_1\in\left(0,\frac{\delta}{24-12\delta}\right)$ and $\eta\le\min\left\{\frac{1}{2L}, \frac{\beta_1}{L}\cdot\sqrt{\frac{\delta}{8}}\right\}$, \sys with momentum SGD converges as
    \begin{align*}
    \frac{1}{T+1}\sum_{t=0}^T \mathbb{E}\!\left[\left\|\nabla f(\mathbf{x}^t)\right\|_2^2\right]
    &\le \frac{8\bigl(f(\mathbf{x}^0)-\inf_{\mathbf{x}} f(\mathbf{x})\bigr)}{\delta\,\eta\,(T+1)}
       + \frac{8\left\|\mathbf{m}^0-\nabla f(\mathbf{x}^0)\right\|_2^2}{\delta\,\beta_1\,(T+1)} \\
    &\quad + \frac{24\,\beta_1\,\sigma^2}{\delta}.
    \end{align*}
\end{theorem}

\begin{corollary}\label{col:1}
	Under Assumptions~\ref{asp:proper}--\ref{asp:contractive}, if we choose $\beta_1=\left(\frac{24}{\delta}+\sigma\sqrt{\frac{\delta^{1/2}\left(T+1\right)}{L\Delta}}\right)^{-1}$, $\eta=\left(2L+ \frac{2^{3/2}L}{\delta^{1/2}\beta_1}\right)^{-1}$, \sys with momentum SGD converges as
    \begin{align}
		&\frac{\sum_{t=0}^{T}\mathbb{E}[\|\nabla f(\mathbf{x}^t)\|_2^2]}{T+1}
		=\mathcal{O}\left(\frac{L\Delta}{\delta^{5/2}(T+1)}+\sqrt{\frac{L\Delta\sigma^2}{\delta^{5/2}(T+1)}}\right),\nonumber
	\end{align}
	where $\Delta:=f(\mathbf{x}^0)-\inf_{\mathbf{x}}f(\mathbf{x})+(\delta/L)\cdot\|\mathbf{m}^0-\nabla f(\mathbf{x}^0)\|_2^2$ (Proofs are in Appendix~\ref{app:proof}).
\end{corollary}

\textbf{Remark.} Corollary~\ref{col:1} yields three key implications. First, it guarantees that the proposed \sys algorithm converges to a stationary solution of problem~\eqref{prob}. Second, it shows that \sys achieves a convergence rate of $\mathcal{O}(1/\sqrt{T})$, matching that of vanilla momentum SGD without gradient quantization. This demonstrates that \sys effectively preserves the valuable gradient information during quantization. Third, the theorem indicates that the convergence rate is affected by the quantization error, quantified by the coefficient $\delta$. This is consistent with our expectations. Since \sys maintains a relatively large $\delta$ (i.e., close to 1), the quantization error remains moderate and does not significantly slow convergence.

\begin{figure}[t]
    \centering
    \begin{subfigure}[b]{0.245\textwidth}
        \centering
        \includegraphics[width=1.0\textwidth]{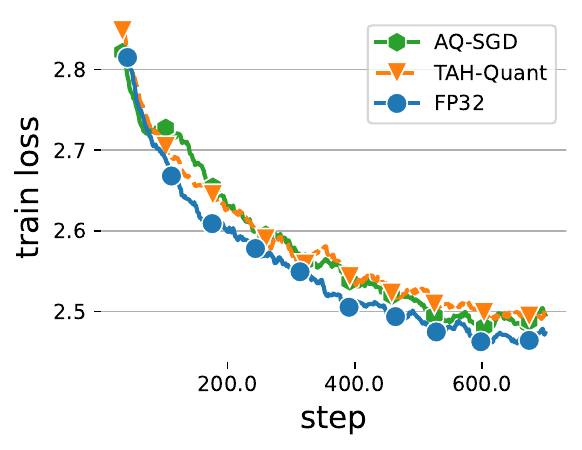}
        \caption{Task (\underline{i}).}
        \label{fig:exp_wiki_conv}
    \end{subfigure}
    \begin{subfigure}[b]{0.245\textwidth}
        \centering
        \includegraphics[width=1.0\textwidth]{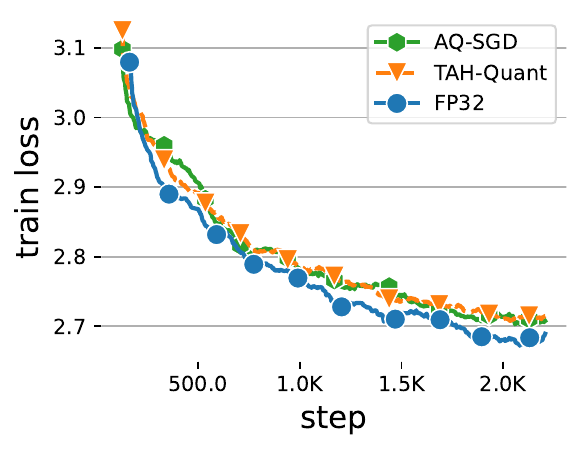}
        \caption{Task (\underline{ii}).}
        \label{fig:exp_arxiv_conv}
    \end{subfigure}
    \begin{subfigure}[b]{0.245\textwidth}
        \centering
        \includegraphics[width=1.0\textwidth]{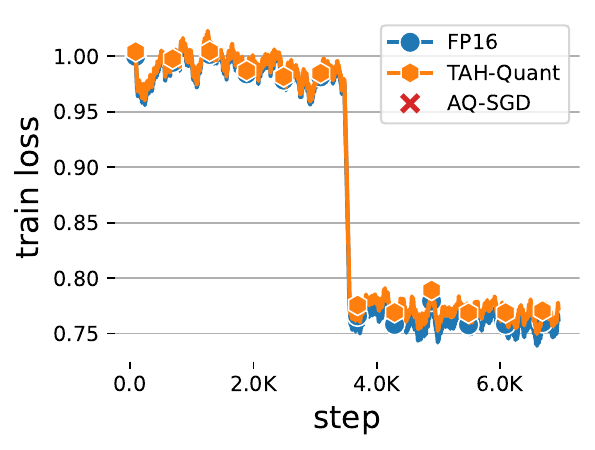}
        \caption{Task (\underline{iii}).}
        \label{fig:exp_magicoder_conv}
    \end{subfigure}
    \begin{subfigure}[b]{0.245\textwidth}
        \centering
        \includegraphics[width=1.0\textwidth]{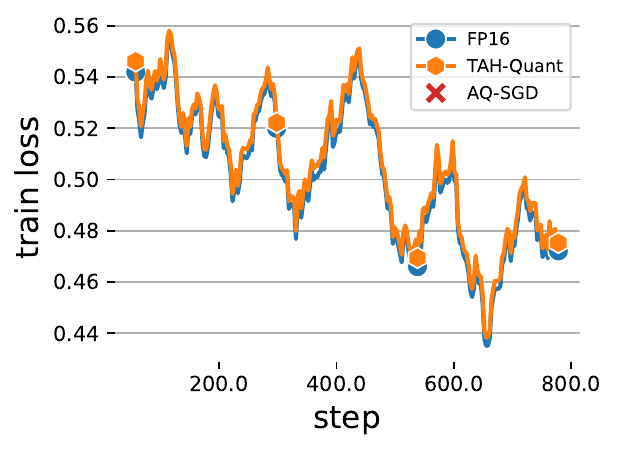}
        \caption{Task (\underline{iv}).}
        \label{fig:exp_platyups_conv}
    \end{subfigure}\\[6pt]
    \begin{subfigure}[b]{0.245\textwidth}
        \centering
        \includegraphics[width=1.0\textwidth]{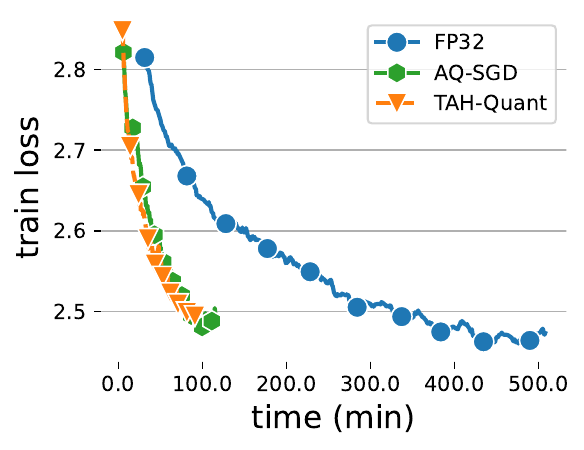}
        \caption{Task (\underline{i}), 100Mbps.}
        \label{fig:exp_platyups_thp_100M}
    \end{subfigure}
    \begin{subfigure}[b]{0.245\textwidth}
        \centering
        \includegraphics[width=1.0\textwidth]{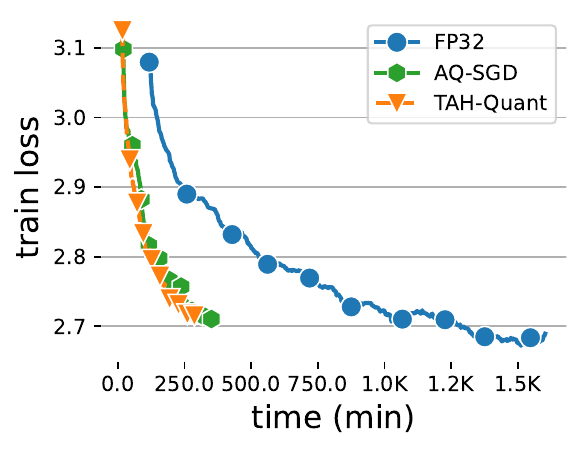}
        \caption{Task (\underline{ii}), 100Mbps.}
        \label{fig:exp_c4_thp_100M}
    \end{subfigure}
    \begin{subfigure}[b]{0.245\textwidth}
        \centering
        \includegraphics[width=1.0\textwidth]{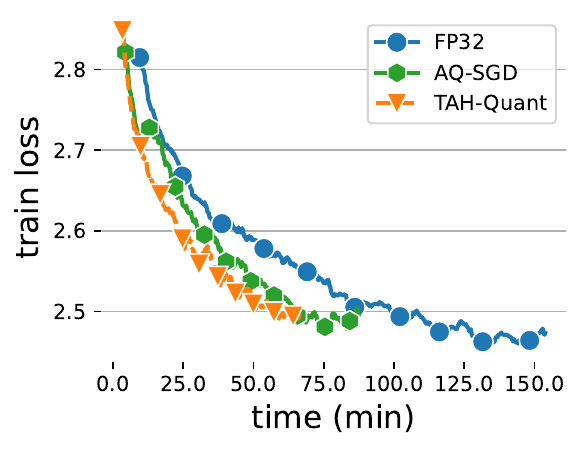}
        \caption{Task (\underline{i}), 500Mbps.}
        \label{fig:exp_platyups_thp_500M}
    \end{subfigure}
    \begin{subfigure}[b]{0.245\textwidth}
        \centering
        \includegraphics[width=1.0\textwidth]{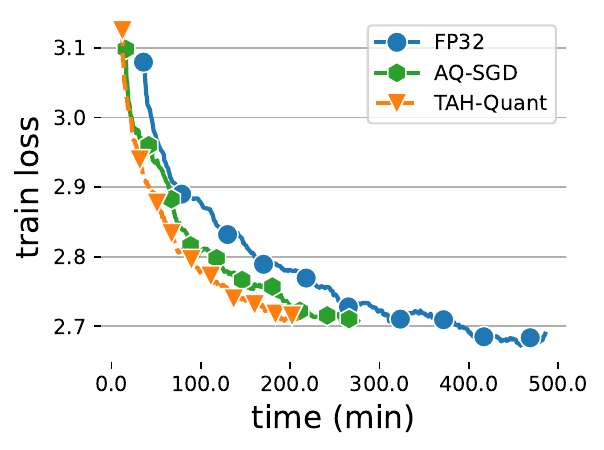}
        \caption{Task (\underline{ii}), 500Mbps.}
        \label{fig:exp_c4_thp_500M}
    \end{subfigure}
    \caption{Top row: training-convergence comparison on the fine-tuning and instruction-tuning tasks (loss vs.\ steps). Bottom row: end-to-end training performance under different network bandwidths (loss vs.\ wall-clock time) on tasks (\underline{i}) and (\underline{ii}). See Section~\ref{sec:exp_setup} for the detailed experimental setup of each task. Pretraining task (\underline{v}) is summarised in Table~\ref{tab:exp_pretrain_llama1b}.}
    \label{fig:exp_convs}
\end{figure}

\section{Evaluation}

We demonstrate that \sys significantly accelerates LLM training over slow network connections. Specifically, we show that: (\underline{i}) on representative benchmarks, \sys enables aggressive quantization of activations and gradients without compromising convergence performance or incurring notable additional system overhead (Section~\ref{sec:exp_e2e}); and (\underline{ii}) the effectiveness of our system design is validated through a series of ablation studies (Section~\ref{sec:ablation}).

\subsection{Experimental Setup}
\label{sec:exp_setup}

\textbf{Datasets and benchmarks.} We evaluate the proposed method on five distinct training scenarios spanning language modeling and instruction-following tasks. Specifically, we fine-tune \texttt{GPT-2XL} (1.5B parameters) on (\underline{i}) \texttt{WikiText-2}, a standard Wikipedia-based language modeling benchmark, and (\underline{ii}) \texttt{ArXiv21}, a corpus of research paper abstracts from arXiv. To assess performance on instruction data, we fine-tune \texttt{Qwen2.5-3B} (3B parameters) on (\underline{iii}) \texttt{Magicoder-Evol-Instruct-110K}, a dataset of 110k code-related instruction-response pairs, and (\underline{iv}) \texttt{Open-Platyups}, a composite open-source instruction tuning dataset covering multiple domains. Finally, we launch from-scratch pretraining of \texttt{LLaMA-3.2-1B} (1.2B parameters) on (\underline{v}) \texttt{SlimPajama-6B}, a deduplicated 6B-token subset of the SlimPajama corpus tokenized with the LLaMA-3.2 tokenizer. These setups cover both general and specialized tasks, as well as supervised instruction tuning and LLM pretraining. We additionally include a supplementary \texttt{LLaMA-8B} pretraining run on \texttt{Proof-Pile} in Appendix~\ref{app:llama8b}.

\textbf{Distributed cluster.} The fine-tuning and instruction-tuning experiments (Tasks (\underline{i})--(\underline{iv})) are conducted on UCloud~\cite{ucloud} using a distributed cluster of 8 instances, each equipped with an Nvidia RTX 3090 GPU. Each model is partitioned into 8 or 4 pipeline stages (one stage per GPU) to execute pipeline parallelism. The cluster's default interconnect bandwidth is 10 Gbps. To emulate slow-network conditions, we throttle inter-instance communication using Linux traffic control (tc), artificially limiting the bandwidth to sub-1 Gbps during training. This setting is common in decentralized environments and is frequently used in prior LLM training evaluations~\cite{kim2025halos, lim2024accelerating, li2024accelerating, erben2023can, wang2023cocktailsgd, borzunov2023distributed}; similar constraints also arise in decentralized inference~\cite{mei2025helix}.

\textbf{Pretraining setup (Task (\underline{v})).} Because the pretraining workload is substantially heavier than fine-tuning, we run Task (\underline{v}) on a single 8-GPU server shared by both the \texttt{FP16} baseline and the \texttt{TAH-Quant} run, each occupying 4 GPUs and configured with pipeline depth PP=4. Sequence length is 4096, the global batch size is 256 (i.e.\ $\sim$1.05\,M tokens / step), and the learning rate follows a cosine decay (peak $1.5{\times}10^{-4}$, min $1{\times}10^{-5}$).

\textbf{Baselines.} We compare our approach with two baseline communication strategies\footnote{Each baseline is integrated into the same pipeline parallel training setup for fair comparison.}.

\begin{itemize}[topsep=2pt, itemsep=1pt, parsep=0pt, partopsep=0pt, leftmargin=*]
\item \texttt{FP32/FP16}, which uses full-precision 32-bit (in Tasks (\underline{i}) and (\underline{ii})) or 16-bit floating point (in Tasks (\underline{iii}), (\underline{iv}), and (\underline{v})) communication with no compression.
\item \texttt{AQ-SGD}, the error-compensated low-bit activation quantization method with theoretical convergence guarantees~\cite{wang2022fine}.
\end{itemize}

\textbf{Default bit allocation.} Unless otherwise specified, \sys uses 80\% INT4 + 20\% INT3 mixed-precision activation quantization, with quantization tile size $G=64$ and allocation-tile size $A=64$. This 4/3-bit setting is the lowest-bit configuration that remains consistently stable in training; more aggressive choices (e.g., INT2) often lead to non-convergence. The 80/20 split balances communication reduction and convergence stability (Appendix~\ref{app:bit_sweep}).

\subsection{End-to-End Performance Results}
\label{sec:exp_e2e}
To systematically evaluate the proposed \sys quantization method, we conduct the experiment and report the corresponding results in terms of training convergence, downstream task performance, and end-to-end training time.

\textbf{Convergence.} Figure~\ref{fig:exp_convs} (top row) and Table~\ref{tab:exp_pretrain_llama1b} together summarize the convergence comparisons across tasks, demonstrating the efficacy and robustness of \sys. Specifically, on tasks (\underline{i}) and (\underline{ii}), where \texttt{AQ-SGD} is executable due to manageable dataset sizes and a multi-epoch training paradigm, \sys achieves comparable or slightly superior convergence compared to \texttt{AQ-SGD}. On larger-scale tasks (i.e., tasks (\underline{iii}), (\underline{iv}), and (\underline{v})), where \texttt{AQ-SGD} becomes infeasible due to prohibitive storage requirements (Task (\underline{iii})) or the single-epoch training constraint (Tasks (\underline{iv}) and (\underline{v})), \sys still closely matches the standard \texttt{FP16} baseline. For from-scratch pretraining of \texttt{LLaMA-3.2-1B} on \texttt{SlimPajama-6B} (Table~\ref{tab:exp_pretrain_llama1b}), \sys matches \texttt{FP16} not only in training loss but also in held-out validation perplexity throughout the entire $\sim$6\,B-token training window, indicating that the proposed activation quantization does not degrade out-of-sample performance.\footnote{We further discuss backward-gradient quantization and numerical stability in Appendix~\ref{app:backward}.}

\begin{table}[h]
    \centering
    \small
    \caption{Pretraining \texttt{LLaMA-3.2-1B} from scratch on \texttt{SlimPajama-6B} (PP=4, sequence length 4096, global batch size 256). Train loss is the rolling mean over a $300$-step window at the indicated token milestone; validation perplexity is measured at the closest evaluation checkpoint. Across the entire $\sim$6\,B-token horizon, \sys is statistically indistinguishable from the uncompressed \texttt{FP16} baseline.}
    \label{tab:exp_pretrain_llama1b}
    \setlength{\tabcolsep}{6pt}
    \begin{tabular}{lcccccccccccc}
    \toprule
    & \multicolumn{6}{c}{\textit{Train loss} ($\downarrow$)} & \multicolumn{6}{c}{\textit{Validation perplexity} ($\downarrow$)} \\
    \cmidrule(lr){2-7} \cmidrule(lr){8-13}
    \textbf{Tokens} & 1\,B & 2\,B & 3\,B & 4\,B & 5\,B & 6\,B & 1\,B & 2\,B & 3\,B & 4\,B & 5\,B & 6\,B \\
    \midrule
    \texttt{FP16}      & 3.778 & 3.326 & 3.149 & 3.032 & 2.954 & 2.893 & 39.18 & 28.54 & 25.36 & 23.03 & 21.42 & 20.67 \\
    \texttt{TAH-Quant} & 3.781 & 3.325 & 3.147 & 3.031 & 2.953 & 2.891 & 38.47 & 28.51 & 25.30 & 22.99 & 21.41 & 20.64 \\
    \bottomrule
    \end{tabular}
\end{table}

\begin{table}[h!]
    \centering
    \small
    \caption{\texttt{Qwen2.5-3B} SFT evaluation. Fine-tuning data are Open-Platyups for ARC/TQ/WG and Magicoder for HE; see Appendix~\ref{app:exp_details}.}
    \label{tab:quality}
    \setlength{\tabcolsep}{10pt}
    \begin{tabular}{lccccc}
    \toprule
        \multirow{2}{*}{Model}
        & \multirow{2}{*}{AVG}
        & \multicolumn{3}{c}{Open-Platyups}
        & \multicolumn{1}{c}{Magicoder} \\
        \cmidrule(lr){3-5} \cmidrule(lr){6-6}
        &       & ARC   & TQ    & WG    & HE    \\
    \midrule
        \texttt{Qwen}     & 51.13 & 47.35 & 48.85 & 68.67 & 39.63 \\
        \texttt{SFT-FP16} & 59.08 & 50.00 & 50.49 & 69.38 & 66.46 \\
        \texttt{SFT-TAH}  & 59.32 & 49.91 & 49.61 & 70.00 & 67.68 \\
    \bottomrule
    \end{tabular}
\end{table}

Furthermore, Table~\ref{tab:quality} reports downstream evaluations for the SFTed \texttt{Qwen2.5-3B} model in Tasks (\underline{iii}) and (\underline{iv}), showing that fine-tuning with \sys preserves model quality and achieves nearly identical performance to \texttt{FP16} across multiple benchmarks. Overall, the results highlight that \sys enables aggressive activation quantization without extra memory overhead, making it broadly applicable and scalable in realistic training scenarios.

\begin{wraptable}{r}{0.42\linewidth}
    \centering
    \captionof{table}{\texttt{GPT2-xl} throughput (tokens / s).}
    \label{tab:throughput}
    \small
    \setlength{\tabcolsep}{3pt}
    \begin{tabular}{lccc}
    \toprule
    \begin{tabular}{@{}c@{}}Network\\Bandwidth\end{tabular}
      & \texttt{FP32}
      & \begin{tabular}{@{}c@{}}\texttt{AQ-SGD}\\{\scriptsize fw4 bw8}\end{tabular}
      & \begin{tabular}{@{}c@{}}\texttt{TAH-Q}\\{\scriptsize fw\textasciitilde4 bw6}\end{tabular} \\
    \midrule
    1Gbps   & 2600 & 4749 & 5650 \\
    500Mbps & 2482 & 4311 & 5749 \\
    300Mbps & 1761 & 4369 & 5120 \\
    100Mbps &  751 & 3310 & 4045 \\
    \bottomrule
    \end{tabular}
\end{wraptable}
\textbf{End-to-end training time.} \sys is designed to preserve per-step convergence while reducing communication time, so its benefit appears primarily in wall-clock efficiency under bandwidth-limited networks rather than in fewer optimization steps. As illustrated in the bottom row of Figure~\ref{fig:exp_convs}, \sys reaches the same loss faster than both uncompressed communication and \texttt{AQ-SGD}, with up to $1.33\times$ wall-clock speedup over \texttt{AQ-SGD}. We attribute the speedup over \texttt{AQ-SGD} to eliminating the activation-cache offloading used by its error-compensation mechanism. Table~\ref{tab:throughput} reports the corresponding \texttt{GPT-2XL} training throughput under 1Gbps, 500Mbps, 300Mbps, and 100Mbps bandwidth; \sys achieves up to $4.3\times$ throughput speedup over uncompressed \texttt{FP32}.

\textbf{Micro-benchmarks.}
To further characterize the impact of system parameters, we sweep network bandwidth and micro batch size while fixing the global batch size to 32, and measure end-to-end throughput (tokens/s). The speedup of \sys is larger under lower bandwidth, and is maximized at moderate micro-batch sizes (mbs=2--4). With very large mbs, the number of micro-batches decreases and the idle time of pipeline bubbles increases, slightly reducing the speedup.
Full results are reported in Appendix~\ref{app:microbench}.

\subsection{Ablation Study}
\label{sec:ablation}

To evaluate the specific contributions of each module in \sys to reduce quantization error under pipeline parallel training, we perform a series of ablation studies and enumerate the experimental results below:

\begin{wraptable}{r}{0.22\linewidth}
    \centering
    \small
    \caption{Ablation: tile-wise quantization group size (TS).}
    \label{tab:abl_twg}
    \setlength{\tabcolsep}{4pt}
    \begin{tabular}{ccc}
    \toprule
    TS  & MMLU  & ARC   \\
    \midrule
    8   & 64.60 & 50.34 \\
    32  & 64.88 & 49.91 \\
    128 & 64.34 & 49.66 \\
    \bottomrule
    \end{tabular}
\end{wraptable}

\underline{First}, to study how the \textbf{tile-wise quantization group size} influences statistical efficiency, we vary the group size to $8$, $32$, and $128$ and compare SFTed \texttt{Qwen2.5-3B} models over the set of benchmarks. Table~\ref{tab:abl_twg} shows that smaller groups (TS=8 and TS=32) achieve nearly identical scores on both MMLU and ARC, whereas TS=128 measurably degrades MMLU (from 64.88 to 64.34) and ARC (from 50.34 to 49.66). We attribute this to the within-group statistical heterogeneity that grows with TS: as the group spans more channels, the single shared scale factor must accommodate increasingly disparate activation magnitudes, inflating the per-element quantization error. We therefore adopt TS=64 as our default, which retains the accuracy of small groups while reducing the per-tile metadata overhead relative to TS=8 by an $8{\times}$ factor.

\underline{Second}, to examine the effectiveness of the \textbf{entropy-guided adaptive bit allocation}, we compare \sys with adaptive bit allocation enabled against a variant without adaptive allocation. Results in Figure~\ref{fig:exp_tla_aba} show that adaptive allocation accelerates training convergence --- the training loss consistently decreases faster with adaptive bit allocation, reflecting reduced quantization error during compression. These observations validate our design of incorporating entropy-based tile-wise bit-width allocation.

\begin{wrapfigure}{r}{0.55\linewidth}
    \centering
    \vspace{-0.5em}
    \begin{subfigure}[b]{0.48\linewidth}
        \centering
        \includegraphics[width=\linewidth]{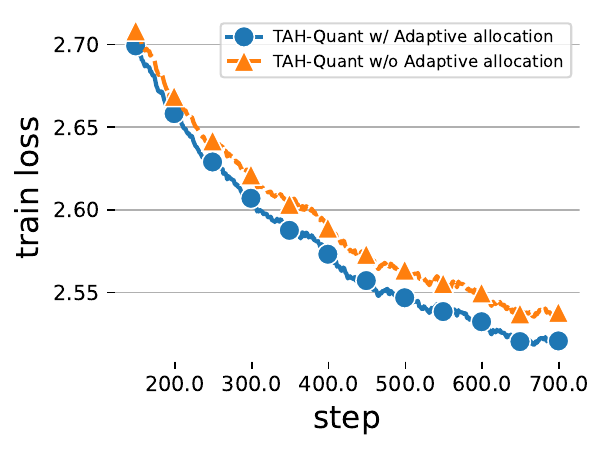}
        \caption{Adaptive bit allocation.}
        \label{fig:exp_tla_aba}
    \end{subfigure}
    \hfill
    \begin{subfigure}[b]{0.48\linewidth}
        \centering
        \includegraphics[width=\linewidth]{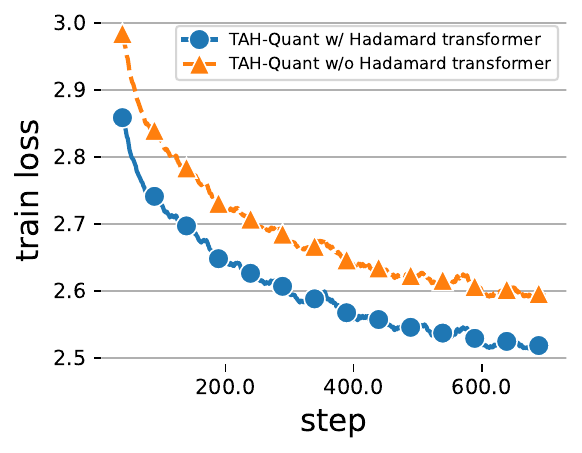}
        \caption{Hadamard transform.}
        \label{fig:exp_ht_aba}
    \end{subfigure}
    \vspace{-0.5em}
    \caption{Ablation studies on \sys components.}
    \label{fig:exp_aba}
    \vspace{-0.5em}
\end{wrapfigure}

\underline{Third}, we evaluate the necessity of the \textbf{Hadamard-based outlier suppression} component in \sys. Comparing training performance between setups with and without the Hadamard transform reveals that including this transform notably improves training stability and convergence speed --- in Figure~\ref{fig:exp_ht_aba}, the training loss remains consistently lower across the training phase when the Hadamard transform is applied. This result indicates that spreading activation energy across coordinates reduces the impact of isolated extreme values, making low-bit quantization less sensitive to outliers. It also highlights the complementary role of pivot element swapping: the swap first relocates the dominant entry to a favorable position, while the Hadamard transform further redistributes its effect, together yielding more stable quantized training.

\section{Conclusion}

We present \sys, a novel activation quantization method that alleviates communication bottlenecks in decentralized pipeline-parallel training of LLMs. \sys integrates fine-grained tile-wise quantization for localized error control, entropy-guided tile-wise bit allocation that refines token-level allocation, and a Hadamard-based transform with pivot swapping to mitigate outliers. We theoretically show that pipeline-parallel training with \sys preserves the same convergence rate ($\mathcal{O}(1/\sqrt{T})$) as standard SGD. Empirical results demonstrate that \sys compresses activations to 3--4 bits without degrading convergence, while improving throughput over uncompressed communication, reducing wall-clock time relative to AQ-SGD, avoiding activation-cache overhead, and maintaining robust generalization across fine-tuning, instruction tuning, and from-scratch pretraining.

\subsubsection*{Broader Impact Statement}
This paper presents work whose goal is to advance the field of machine learning. There might be some potential societal consequences of our work, none of which we feel must be specifically highlighted here.

\appendix

\section{The Use of LLMs in Writing}
We used a LLM, namely \textsc{OpenAI-GPT5}, to polish the writing of this manuscript. No other generative AI functionality is used in the writing of this submission.

\section{Experimental Details}\label{app:exp_details}

\textbf{Downstream evaluation protocol (Table~\ref{tab:quality}).} We evaluate the SFTed \texttt{Qwen2.5-3B} models on four standard benchmarks: ARC~\cite{DBLP:journals/corr/abs-1803-05457}, TruthfulQA (TQ)~\cite{lin-etal-2022-truthfulqa}, WinoGrande (WG)~\cite{sakaguchi2021winogrande}, and HumanEval (HE)~\cite{chen2021codex}. All evaluations are conducted in a zero-shot setting using the lm-evaluation-harness~\cite{eval-harness} framework. We report normalized accuracy for ARC-Challenge, accuracy for WinoGrande, mc2 for TruthfulQA, and pass@1 for HumanEval.
\textbf{Fine-tuning.} We fine-tune the \texttt{GPT2-xl} on \texttt{WikiText-2} and \texttt{ArXiv21} for 10 epochs. Specifically, we set the learning rate to 5.0e-6, the batch size to 32, and the micro-batch size to 1, max sequence length to 1024 for both datasets. The learning rate decays linearly after the warm-up stage.

\textbf{Instruction-tuning.} We perform instruction tuning on \texttt{Qwen2.5-3B} using \texttt{Open-Platyups} and \texttt{Magicoder-110K} for 1 and 2 epochs, respectively. The learning rate is set to 2.0e-5, with a batch size of 32 for both datasets. We use a cosine learning rate scheduler for \texttt{Open-Platyups}, and a cosine scheduler with a minimum learning rate of 2.0e-6 for \texttt{Magicoder-110K}.

\textbf{Pretraining (Task (\underline{v})).}
We pretrain \texttt{LLaMA-3.2-1B} from scratch on the \texttt{SlimPajama-6B} corpus (a deduplicated 6\,B-token subset of SlimPajama) tokenized with the LLaMA-3.2 tokenizer. The model has 16 transformer layers, hidden size 2048, 32 attention heads, 8 KV heads, and FFN hidden size 8192; we use sequence length 4096 and a global batch size of 256 (i.e.\ $\sim$1.05\,M tokens / step). The peak learning rate is $1.5{\times}10^{-4}$ with cosine decay to a minimum of $1{\times}10^{-5}$, weight decay $0.1$, $\beta_1{=}0.9$, $\beta_2{=}0.95$, gradient clipping at 1.0, and a small linear warmup. Pipeline depth is PP=4 (one transformer stage per GPU). The \texttt{FP16} baseline and the \texttt{TAH-Quant} arm run concurrently on the same 8-GPU server (4 GPUs each), sharing the same hardware and wall-clock budget. The main text reports training through $\sim$6\,B tokens, where both arms reach the same loss and validation perplexity within noise.

A supplementary \texttt{LLaMA-8B} pretraining run on \texttt{Proof-Pile} is reported in Appendix~\ref{app:llama8b}: $20k$ iterations at global batch size $131{,}072$ tokens (i.e.\ $\sim$2.5\,B tokens, $>$30\% of the corpus), peak learning rate $1.5{\times}10^{-4}$ with cosine decay to $1{\times}10^{-5}$, and weight decay 0.1.

\section{Missing Proofs}\label{app:proof}

In this section, we provide detailed proofs for Theorem~\ref{thm:quantize}. We first prove the following lemma.
\begin{lemma}[Descent lemma]\label{descent_lemma}
	Under Assumption~\ref{asp:smoothness} and the update rule~\ref{adaptrule}, it holds that
	\begin{align}
		f(\mathbf{x}^{t+1})\le&	f(\mathbf{x}^{t})-\frac{\eta}{2}\|\nabla f(\mathbf{x}^t)\|_2^2-\left(\frac{1}{2\eta}-\frac{L}{2}\right)\|\mathbf{x}^{t+1}-\mathbf{x}^t\|_2^2+\frac{\eta}{2}\|\mathbf{m}^t-\nabla f(\mathbf{x}^t)\|_2^2.
	\end{align}
\end{lemma}
\begin{proof}
	by Assumption~\ref{asp:smoothness} we have
	\begin{align}
		f(\mathbf{x}^{t+1})\le&f(\mathbf{x}^{t})+\eta\langle\nabla f(\mathbf{x}^t),\frac{1}{\eta}(\mathbf{x}^{t+1}-\mathbf{x}^t)\rangle+\frac{L}{2}\|\mathbf{x}^{t+1}-\mathbf{x}^t\|_2^2\nonumber\\
		=& f(\mathbf{x}^{t})-\frac{\eta}{2}\|\nabla f(\mathbf{x}^t)\|_2^2-\frac{1}{2\eta}\|\mathbf{x}^{t+1}-\mathbf{x}^t\|_2^2+\frac{\eta}{2}\|\nabla f(\mathbf{x}^t)-\mathbf{m}^t\|_2^2+\frac{L}{2}\|\mathbf{x}^{t+1}-\mathbf{x}^t\|_2^2.
	\end{align}
	where the second equality uses $2\langle a,b \rangle =\|a\|_2^2+\|b\|_2^2
	-\|a-b\|_2^2$
\end{proof}
\begin{lemma}[momentum contraction]\label{lm:moment_cont}
	Under Assumptions~\ref{asp:proper}--\ref{asp:contractive}, if $\delta\in(0,1)$, it holds that
	\begin{align}
		\mathbb{E}[\|\mathbf{m}^t-\nabla f(\mathbf{x}^t)\|_2^2]\le&\left(1-\beta_1\left(1-\frac{\delta}{2}\right)\right)\mathbb{E}[\|\mathbf{m}^{t-1}-\nabla f(\mathbf{x}^{t-1})\|_2^2]+\frac{2L^2}{\delta\beta_1}\mathbb{E}[\|\mathbf{x}^t-\mathbf{x}^{t-1}\|_2^2]\nonumber\\
		&+(\beta_1+6\beta_1^2)(1-\delta)\mathbb{E}[\|\nabla f(\mathbf{x}^t)\|_2^2]+3(2-\delta)\beta_1^2\sigma^2.
	\end{align}
\end{lemma}
\begin{proof}
	According to the update of momentum~\ref{adaptrule}, we have
	\begin{align}
		\mathbf{m}^{t}-\nabla f(\mathbf{x}^{t})=&(1-\beta_1)(\mathbf{m}^{t-1}-\nabla f(\mathbf{x}^{t-1})+\nabla f(\mathbf{x}^{t-1})-\nabla f(\mathbf{x}^t))+\beta_1(\hat{\mathbf{g}}^t-\nabla f(\mathbf{x}^t)).\nonumber
	\end{align}
	Taking expectation we have
	\begin{align}
		\mathbb{E}[\|\mathbf{m}^t-\nabla f(\mathbf{x}^t)\|_2^2]=&\mathbb{E}[\|(1-\beta_1)(\mathbf{m}^{t-1}-\nabla f(\mathbf{x}^{t-1})+\nabla f(\mathbf{x}^{t-1})-\nabla f(\mathbf{x}^t))+\beta_1(\mathbb{E}[\hat{\mathbf{g}}^t]-\nabla f(\mathbf{x}^t))\|_2^2]\nonumber\\
		&+\beta_1^2\mathbb{E}[\|\hat{\mathbf{g}}^t-\mathbb{E}[\hat{\mathbf{g}}^t]\|_2^2].\label{eq:pflm-m-1}
	\end{align}
	For the first term, applying Jensen's inequality yields
	\begin{align}
		&\mathbb{E}[\|(1-\beta_1)(\mathbf{m}^{t-1}-\nabla f(\mathbf{x}^{t-1})+\nabla f(\mathbf{x}^{t-1})-\nabla f(\mathbf{x}^t)+\beta_1(\mathbb{E}[\hat{\mathbf{g}}^t]-\nabla f(\mathbf{x}^t))\|_2^2]\nonumber\\
		\le&(1-\beta_1)\mathbb{E}[\|\mathbf{m}^{t-1}-\nabla f(\mathbf{x}^{t-1})+\nabla f(\mathbf{x}^{t-1})-\nabla f(\mathbf{x}^t)\|_2^2]+\beta_1\mathbb{E}[\|\mathbb{E}[\hat{\mathbf{g}}^t]-\nabla f(\mathbf{x}^t)\|_2^2].\label{eq:pflm-m-2}
	\end{align}
	By Young's inequality, we have
	\begin{align}
		\mathbb{E}[\|\mathbf{m}^{t-1}-\nabla f(\mathbf{x}^{t-1})+\nabla f(\mathbf{x}^{t-1})-\nabla f(\mathbf{x}^t)\|_2^2]\le&\left(1+\frac{\delta\beta_1}{2}\right)\mathbb{E}[\|\mathbf{m}^{t-1}-\nabla f(\mathbf{x}^{t-1})\|_2^2]\nonumber\\
		&+\left(1+\frac{2}{\delta\beta_1}\right)\mathbb{E}[\|\nabla f(\mathbf{x}^t)-\nabla f(\mathbf{x}^{t-1})\|_2^2].\label{eq:pflm-m-3}
	\end{align}
	For the second term, applying Cauchy's inequality yields
	\begin{align}
		\mathbb{E}[\|\hat{\mathbf{g}}^t-\mathbb{E}[\hat{\mathbf{g}}^t]\|_2^2]\le&3\mathbb{E}\|\hat{\mathbf{g}}^t-\mathbf{g}^t\|_2^2+3\mathbb{E}[\|\mathbf{g}^t-\nabla f(\mathbf{x}^t)\|_2^2]+3\mathbb{E}[\|\nabla f(\mathbf{x}^t)-\mathbb{E}[\hat{\mathbf{g}}^t]\|_2^2]\nonumber\\
		\le&3(1-\delta)\mathbb{E}[\|\nabla f(\mathbf{x}^t)\|_2^2]+3(1-\delta)\mathbb{E}[\| \mathbf{g}^t\|_2^2]+3\sigma^2,\nonumber\\
		\le&6(1-\delta)\mathbb{E}[\|\nabla f(\mathbf{x}^t)\|_2^2]+3(2-\delta)\sigma^2,\label{eq:pflm-m-4}
	\end{align}
	where the inequality uses Assumption~\ref{asp:stochastic} and~\ref{asp:contractive}.
	Applying~\eqref{eq:pflm-m-2}\eqref{eq:pflm-m-3}\eqref{eq:pflm-m-4} to~\eqref{eq:pflm-m-1} and using Assumption~\ref{asp:smoothness} and~\ref{asp:contractive}, we obtain~\ref{lm:moment_cont}.
\end{proof}

\textbf{Remark.} From this proof, it is evident that both inequalities in Assumption~\ref{asp:contractive} are necessary. In particular, the second inequality is essential for bounding the variance of $\hat{\mathbf{g}}^t$, which plays a crucial role in the overall convergence analysis.

Now we are ready to prove Theorem~\ref{thm:quantize}. We first restate the theorem in Theorem~\ref{thm:quantize-restate}.

\begin{theorem}\label{thm:quantize-restate}
	Under Assumptions~\ref{asp:proper}--\ref{asp:contractive}, if $\beta_1\in(0,\delta/(24-12\delta))$, $\delta_1\in(0,1)$ and $\eta\le\min\{1/2L,\sqrt{(\delta\beta_1^2)/(8L^2)}\}$, \sys with momentum SGD converges as
	\begin{align}
		\frac{1}{T+1}\sum_{t=0}^T\mathbb{E}[\|\nabla f(\mathbf{x}^t)\|_2^2]\le&\frac{8[f(\mathbf{x}^0)-\inf_{\mathbf{x}}f(\mathbf{x})]}{\delta\eta(T+1)}+\frac{8\|\mathbf{m}^0-\nabla f(\mathbf{x}^0)\|_2^2}{\delta\beta_1(T+1)}+\frac{24\beta_1\sigma^2}{\delta}.\label{eq:thm-restate}
	\end{align}
\end{theorem}
\begin{proof}
	By Lemma~\ref{descent_lemma}, we have
	\begin{align}
		f(\mathbf{x}^{t+1})-f(\mathbf{x}^t)\le&-\left(\frac{1}{2\eta}-\frac{L}{2}\right)\|\mathbf{x}^{t+1}-\mathbf{x}^t\|_2^2+\frac{\eta}{2}\|\nabla f(\mathbf{x}^t)-\mathbf{m}^t\|_2^2-\frac{\eta}{2}\|\nabla f(\mathbf{x}^t)\|_2^2.\label{eq:pfthm-1}
	\end{align}
	Taking expectation and summing~\eqref{eq:pfthm-1} for $t=0,1,\cdots,T$ yields
	\begin{align}
		\inf_{\mathbf{x}}f(\mathbf{x})-f(\mathbf{x}^0)\le&\frac{\eta}{2}\sum_{t=0}^{T}\mathbb{E}[\|\nabla f(\mathbf{x}^t)-\mathbf{m}^t\|_2^2]-\left(\frac{1}{2\eta}-\frac{L}{2}\right)\sum_{t=0}^{T}\mathbb{E}[\|\mathbf{x}^{t+1}-\mathbf{x}^t\|_2^2]\nonumber\\
		&-\frac{\eta}{2}\sum_{t=0}^T\mathbb{E}[\|\nabla f(\mathbf{x}^t)\|_2^2].\label{eq:pfthm-2}
	\end{align}
	summing the inequality in Lemma~\ref{lm:moment_cont} for $t=1,2,\cdots,T$ we have
    \begin{align}
    	\beta_1\left(1-\frac{\delta}{2}\right)\sum_{t=0}^T\mathbb{E}[\|\mathbf{m}^t-\nabla f(\mathbf{x}^t)\|_2^2]\le& \|\mathbf{m}^0-\nabla f(\mathbf{x}^0)\|_2^2+\frac{2L^2}{\delta\beta_1}\sum_{t=1}^T\|\mathbf{x}^t-\mathbf{x}^{t-1}\|_2^2\nonumber\\
    	&+\left(1-\delta\right)(\beta_1+6\beta_1^2)\sum_{t=1}^T\mathbb{E}[\|\nabla f(\mathbf{x}^t)\|_2^2]+3T(2-\delta)\beta_1^2\sigma^2.\label{eq:m_sum1}
    \end{align}
	 noting that $\delta\in(0,1)$ we obtain
	 \begin{align}
		\sum_{t=0}^T\mathbb{E}[\|\mathbf{m}^t-\nabla f(\mathbf{x}^t)\|_2^2]\le&\frac{2\|\mathbf{m}^0-\nabla f(\mathbf{x}^0)\|_2^2}{\beta_1}+\frac{4L^2}{\delta\beta_1^2}\sum_{t=1}^T\|\mathbf{x}^t-\mathbf{x}^{t-1}\|_2^2\nonumber\\
		&+\left(1-\frac{\delta}{2}\right)(1+6\beta_1)\sum_{t=1}^T\mathbb{E}[\|\nabla f(\mathbf{x}^t)\|_2^2]+6T\beta_1\sigma^2.\label{eq:m_sum}
	\end{align}
	Applying~\ref{eq:m_sum} to~\eqref{eq:pfthm-2} and noting that $\beta_1\in(0,\delta/(24-12\delta))$ implies $(1-\delta/2)(1+6\beta_1)\le1-\delta/4$, we obtain
	\begin{align}
		\frac{1}{T+1}\sum_{t=0}^T\mathbb{E}[\|\nabla f(\mathbf{x}^t)\|_2^2]\le&\frac{8[f(\mathbf{x}^0)-\inf_{\mathbf{x}}f(\mathbf{x})]}{\delta\eta(T+1)}+\frac{8\|\mathbf{m}^0-\nabla f(\mathbf{x}^0)\|_2^2}{\delta\beta_1(T+1)}+\frac{24\beta_1\sigma^2}{\delta}\nonumber\\
		&-\frac{8}{\delta\eta}\left(\frac{1}{2\eta}-\frac{L}{2}-\frac{2\eta L^2}{\delta\beta_1^2}\right)\sum_{t=0}^T\|\mathbf{x}^{t+1}-\mathbf{x}^t\|_2^2.\label{eq:pfthm-3}
	\end{align}
	Since $\eta\le\min\{1/2L,\sqrt{(\delta\beta_1^2)/(8L^2)}\}$ implies $1/(4\eta)\ge L/2$ and $1/(4\eta)\ge(2\eta L^2)/(\delta\beta_1^2)$,~\eqref{eq:thm-restate} is a direct result of~\eqref{eq:pfthm-3}.
\end{proof}

\section{Supplementary Pretraining: \texttt{LLaMA-8B} on \texttt{Proof-Pile}}
\label{app:llama8b}

For completeness we also include a from-scratch pretraining run of \texttt{LLaMA-8B} on the \texttt{Proof-Pile} mathematical corpus, using the configuration described in the ``Pretraining'' paragraph of Appendix~\ref{app:exp_details} ($20k$ iterations, global batch size $131{,}072$ tokens, $\sim$2.5\,B tokens total, peak learning rate $1.5{\times}10^{-4}$ with cosine decay, weight decay $0.1$, PP=4).

We present this larger-model run as a supplementary scale check and use the smaller \texttt{LLaMA-3.2-1B} run as the main-text pretraining benchmark. The reason is statistical efficiency rather than model size alone: under Chinchilla-style scaling laws~\cite{hoffmann2022training}, the \texttt{LLaMA-3.2-1B} run receives substantially more tokens per parameter ($\sim$5) than the \texttt{LLaMA-8B} run ($\sim$0.3), making it the more informative setting for detecting mature pretraining-quality differences under the available token budget.

Figure~\ref{fig:exp_llama8b_proofpile} reports the training loss and the held-out validation perplexity for this run. The training-loss curve (Figure~\ref{fig:exp_llama8b_loss}) covers the full $\sim$0.4--2.5\,B-token horizon after the initial warmup transient is dropped (matching the slicing convention used in~\cite{wang2022fine}). The validation-PPL curve (Figure~\ref{fig:exp_llama8b_ppl}) shows perplexity at $0.5\,$B-spaced checkpoints over the $0.5$--$2.5\,$B-token range. Across the entire horizon, \sys is essentially indistinguishable from the uncompressed \texttt{FP16} baseline: the validation-PPL gap stays within $\pm 0.013$ at every measured checkpoint and shows no upward drift, indicating that the per-step quantization error of \sys does not visibly accumulate even at the 8B-parameter scale.

\begin{figure}[h]
    \centering
    \begin{subfigure}[b]{0.46\textwidth}
        \centering
        \includegraphics[width=1.0\textwidth]{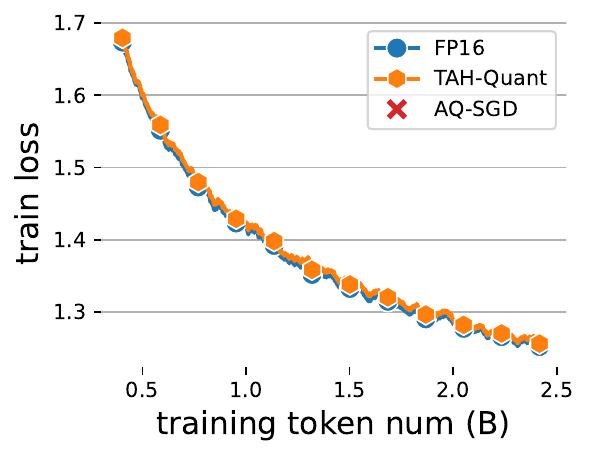}
        \caption{Training loss vs.\ tokens.}
        \label{fig:exp_llama8b_loss}
    \end{subfigure}
    \hfill
    \begin{subfigure}[b]{0.46\textwidth}
        \centering
        \includegraphics[width=1.0\textwidth]{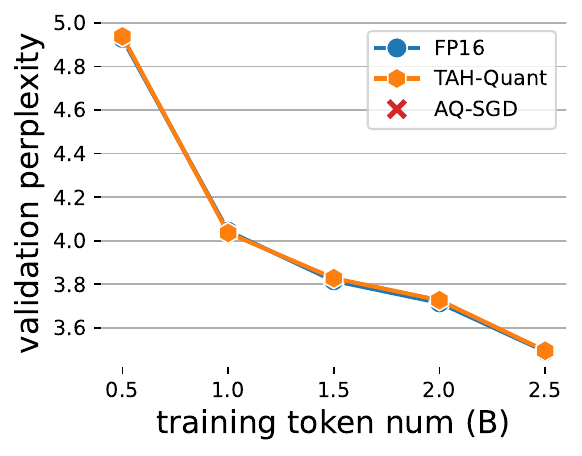}
        \caption{Validation perplexity vs.\ tokens.}
        \label{fig:exp_llama8b_ppl}
    \end{subfigure}
    \caption{Pretraining \texttt{LLaMA-8B} from scratch on \texttt{Proof-Pile} (PP=4, global batch size $32$, sequence length $4096$). \sys (orange) tracks the uncompressed \texttt{FP16} baseline (blue) on both training loss (a) and held-out validation perplexity (b). \texttt{AQ-SGD} is shown in the legend for reference but is infeasible on this single-epoch pretraining workload (its activation cache would require storing per-microbatch forward activations for the entire corpus).}
    \label{fig:exp_llama8b_proofpile}
\end{figure}

\section{Micro-benchmarks: Sweeping Micro-batch Size and Bandwidth}
\label{app:microbench}

We fix the global batch size to 32 and sweep the micro-batch size (mbs) and inter-stage bandwidth.
Tables~\ref{tab:microbench-500}--\ref{tab:microbench-1g} report throughput (tokens/s) and speedup.

\begin{table}[H]
\centering
\caption{Throughput and speedup under 500 Mbps (global batch = 32).}
\label{tab:microbench-500}
\begin{tabular}{lccc}
\toprule
 tokens / s& w/o \sys & w/ \sys & Speedup \\
\midrule
mbs = 1 & 2406 & 4289 & 1.78$\times$ \\
mbs = 2 & 1710 & 4016 & 2.35$\times$ \\
mbs = 4 & 1317 & 3023 & 2.30$\times$ \\
mbs = 8 & 981  & 2099 & 2.14$\times$ \\
\bottomrule
\end{tabular}
\end{table}

\begin{table}[H]
\centering
\caption{Throughput and speedup under 1 Gbps (global batch = 32).}
\label{tab:microbench-1g}
\begin{tabular}{lccc}
\toprule
 tokens / s& w/o \sys & w/ \sys & Speedup \\
\midrule
mbs = 1 & 3449 & 5177 & 1.50$\times$ \\
mbs = 2 & 2576 & 4655 & 1.81$\times$ \\
mbs = 4 & 2003 & 3475 & 1.73$\times$ \\
mbs = 8 & 1485 & 2436 & 1.64$\times$ \\
\bottomrule
\end{tabular}
\end{table}

\textbf{Discussion.}
\sys yields larger speedups under lower bandwidth since inter-stage communication dominates (up to 2.35$\times$ at 500 Mbps).
Speedup increases with micro-batch size and peaks at mbs=2--4, where communicated activation volume per step is larger.
At a very large micro-batch size (mbs=8), increased pipeline idle time slightly reduces the speedup.

\section{Sensitivity of outlier threshold $\tau$}\label{app:tau}
We evaluate the robustness of the outlier detection threshold $\tau$ in Eq.~\ref{eq:tau} by varying $\tau$.
$\tau=0$ corresponds to always applying the transform (treating every tile as containing outliers), while $\tau=\infty$ corresponds to never applying it.
Table~\ref{tab:tau_ablation} shows that $\tau=2.0$ yields the most stable and fastest loss decrease.

\begin{table}[h]
\centering
\small
\begin{tabular}{lccccc}
\toprule
Loss & Step 100 & Step 200 & Step 300 & Step 400 & Step 500 \\
\midrule
$\tau=0$ (always transform) & 2.97 & 2.79 & 2.63 & 2.60 & 2.60 \\
$\tau=2$ (ours)             & \textbf{2.72} & \textbf{2.63} & \textbf{2.59} & \textbf{2.56} & \textbf{2.55} \\
$\tau=4$                    & 2.73 & 2.65 & 2.61 & 2.58 & 2.57 \\
$\tau=\infty$ (never transform) & 2.82 & 2.72 & 2.68 & 2.64 & 2.63 \\
\bottomrule
\end{tabular}
\caption{Effect of the outlier detection threshold $\tau$ on training loss.}
\label{tab:tau_ablation}
\end{table}

\section{Additional Bit-Allocation Studies}
\label{app:bit_sweep}

\subsection{Varying INT4/INT3 Ratios}
Table~\ref{tab:int4_int3_ratio} reports training losses under different INT4/INT3 splitting ratios. The mixed setting lies between the two extremes, and increasing the INT4 fraction yields more stable convergence in this sweep.

\begin{table}[H]
\centering
\small
\setlength{\tabcolsep}{4pt}
\begin{tabular}{lcccccccc}
\toprule
\textbf{Loss} & step 50 & step 100 & step 150 & step 200 & step 300 & step 400 & step 500 & step 600 \\
\midrule
100\% INT4 & 2.78 & 2.70 & 2.66 & 2.61 & 2.56 & 2.52 & 2.51 & 2.49 \\
50\% INT4 + 50\% INT3 & 2.82 & 2.74 & 2.71 & 2.66 & 2.62 & 2.58 & 2.57 & 2.55 \\
100\% INT3 & 2.85 & 2.77 & 2.74 & 2.69 & 2.65 & 2.62 & 2.60 & 2.58 \\
\bottomrule
\end{tabular}
\caption{Loss trajectories under different INT4/INT3 ratio settings.}
\label{tab:int4_int3_ratio}
\end{table}

\subsection{Extremely Low-bit Configuration with INT2}
We additionally evaluate a more aggressive configuration involving INT2. Table~\ref{tab:int2_instability} shows that the INT2-involved setting exhibits degraded behavior in this experiment, consistent with instability under extremely low-bit quantization.

\begin{table}[H]
\centering
\small
\setlength{\tabcolsep}{6pt}
\begin{tabular}{lccccccc}
\toprule
\textbf{Loss} & step 0 & step 100 & step 200 & step 300 & step 400 & step 500 & step 600 \\
\midrule
80\% INT2 + 20\% INT3 & 3.31 & 3.22 & 3.20 & 3.19 & 3.17 & 3.18 & 3.15 \\
\bottomrule
\end{tabular}
\caption{Loss trajectories with INT2-involved quantization, where training becomes unstable.}
\label{tab:int2_instability}
\end{table}

\section{Backward Gradients under Ultra-low Precision}
\label{app:backward}

In the main text, we quantize backward gradients with a higher-bit fixed-point compressor to strike a practical balance between system efficiency and numerical stability.
Here we stress-test backward quantization under an ultra-low precision budget (4-bit).
As shown in Table~\ref{tab:backward-4bit}, the naive 4-bit fixed-point scheme diverges, whereas applying \sys to backward gradients remains stable and continues to converge.

\begin{table}[H]
\centering
\small
\begin{tabular}{lccccccc}
\toprule
 \textbf{Loss} & step 0 & step 100 & step 200 & step 300 & step 400 & step 500 & step 600 \\
\midrule
\texttt{\sys-4bit}   & 2.87 & 2.71 & 2.63 & 2.59 & 2.56 & 2.55 & 2.52 \\
\texttt{NAIVE-4bit}  & 2.87 & 3.04 & 3.16 & 3.21 & 3.49 & 4.53 & 6.04 \\
\bottomrule
\end{tabular}
\caption{Training loss with 4-bit backward gradient quantization.}
\label{tab:backward-4bit}
\end{table}

\paragraph{Numerical stability in long-horizon training.}
During long-horizon pretraining, we observed a single instability event where the loss became \texttt{NaN} around 1.9B tokens when using an 8-bit fixed-point compressor for backward gradients.
Our investigation indicates that the issue is not caused by the forward \sys design (activation quantization), but by insufficient numerical precision in the backward fixed-point path.
Increasing the backward precision from 8-bit to 10-bit and resuming from the checkpoint eliminated the issue and restored stable training.
This motivates studying more robust backward quantization under aggressive bit budgets; as shown in Table~\ref{tab:backward-4bit}, \sys remains stable even at 4-bit, whereas naive fixed-point quantization diverges.

\paragraph{Gradient scaling.}
Gradient magnitudes can be small in practice (e.g., on the order of $10^{-5}$ in our setting), which makes ultra-low-bit quantization particularly sensitive to rounding and underflow.
To improve numerical robustness, we rescale gradients before quantization and invert the scaling after dequantization:
\begin{equation}
\hat{\mathbf{g}} = Q(c\mathbf{g})/c,
\end{equation}
where $Q(\cdot)$ denotes the chosen quantizer and $c$ is a constant scaling factor.

\bibliographystyle{unsrt}
\bibliography{ref}

\end{document}